\documentclass{article}

\usepackage[preprint]{neurips_2026}

\usepackage[utf8]{inputenc} %
\usepackage[T1]{fontenc}    %
\usepackage{hyperref}       %
\usepackage{url}            %
\usepackage{booktabs}       %
\usepackage{amsfonts}       %
\usepackage{nicefrac}       %
\usepackage{microtype}      %
\usepackage{xcolor}         %

\usepackage[utf8]{inputenc} %
\usepackage[T1]{fontenc}    %
\usepackage{hyperref}       %
\usepackage{url}            %
\usepackage{booktabs}       %
\usepackage{amsfonts}       %
\usepackage{nicefrac}       %
\usepackage{microtype}      %
\usepackage{xcolor}         %

\usepackage{bm} %
\usepackage{amsmath} 
\usepackage{amssymb}
\usepackage{graphicx}
\usepackage{amsthm}
\usepackage[shortlabels]{enumitem}
\usepackage{todonotes}
\usepackage[capitalize,noabbrev]{cleveref}
\usepackage{algorithm}
\usepackage{algpseudocode}
\usepackage{subcaption}
\usepackage{wrapfig}

\definecolor{tabblue}{HTML}{56B4E9}
\definecolor{taborange}{HTML}{E69F00}
\newcommand{\legendline}[1]{\textcolor{#1}{\rule[0.45ex]{1.8ex}{1.4pt}}}
\newcommand{\legendmark}[2]{%
  \textcolor{#1}{%
    \rule[0.5ex]{0.6ex}{1.2pt}%
    \kern0.15ex\raisebox{-0.1ex}{\scalebox{0.7}{$#2$}}\kern0.15ex%
    \rule[0.5ex]{0.6ex}{1.2pt}%
  }%
}
\newcommand{\colorours}{\legendline{tabblue}}
\newcommand{\colormh}{\legendline{taborange}}
\usepackage{mwe} %
\usepackage{tikz}

\renewcommand{\epsilon}{\varepsilon}

\def\setof#1{\left\{#1  \right\}}
\newcommand{\sep}{\,\vert \,}

\def\abs#1{\left|#1  \right|}
\def\norm#1{\left\| #1 \right\|}
\def\indicator#1{\mathbf{1}\left\{ #1 \right\}}

\renewcommand*\d{\mathop{}\!\mathrm{d}}

\newcommand{\xtest}{x_{\mathrm{test}}}
\newcommand{\ytest}{y_{\mathrm{test}}}
\theoremstyle{plain}
\newtheorem{theorem}{Theorem}[section]
\newtheorem{proposition}[theorem]{Proposition}
\newtheorem{lemma}[theorem]{Lemma}
\newtheorem{corollary}[theorem]{Corollary}
\theoremstyle{definition}

\newtheorem{observation}{Observation}
\newtheorem{fact}{Fact}
\theoremstyle{remark}

\renewcommand{\P}{\mathbb{P}}
\newcommand{\E}{\mathbb{E}}

\def\calF{\mathcal{F}}

\def\calV{\mathcal{V}}
\def\calX{\mathcal{X}}

\newcommand\cc{\boldsymbol{\mathit{c}}}

\renewcommand\ss{\boldsymbol{\mathit{s}}}

\newcommand\uu{\boldsymbol{\mathit{u}}}

\newcommand{\covdist}{P_X}
\newcommand{\prior}{p_{\text{LLM}}}
\newcommand{\dcal}{D_\text{cal}}

\newcommand{\xhdr}[1]{\vspace{0mm}\noindent{{\bf #1.}}}

\usepackage{hyperref}
\hypersetup{
  colorlinks = true, 
  urlcolor = {blue!50!black},
  linkcolor = {blue},
  citecolor = {green!50!black}
}

\title{Conformal Language Modeling via Posterior Sampling}

\author{%
  Nicolas Emmenegger\thanks{Equal contribution.}
  \quad
  Theo X.\ Olausson$^*$
  \quad
  Armando Solar-Lezama
  \quad
  Chara Podimata
  \\
  Massachusetts Institute of Technology
  \\
  \texttt{\{nemm,theoxo,asolar,podimata\}@mit.edu}
}

\begin{document}

\maketitle

\begin{abstract}
  Large Language Models remain plagued by hallucinations.
Recent work has sought to tame their prevalence using statistical techniques based on conformal prediction, with both theoretical and empirical success.
However, these methods operate in a post-hoc fashion,
treating the sampling procedure itself as atomic and 
then surgically altering samples to remove hallucinated claims.
This disconnect between filtering and generation can result in samples that are incoherent, inconsistent, or simply unlikely under the model itself. Moreover, post-hoc surgery is unable to shift probability mass towards more useful and helpful responses.
To address these issues, we propose to instead sample from approximations to an LLM posterior, where the conditioning event corresponds to a calibrated, high-scoring region.
We develop a calibration procedure tailored to the setting of conditional sequential generation that effectively identifies this region and achieves target risk control.
Empirically, we apply our method to case studies focused on open-ended biography generation and mathematical problem solving; compared to prior work, we obtain the same statistical guarantees, with higher downstream utility.
\end{abstract}

\section{Introduction}
\label{section:intro}

Large Language Models (LLMs) are increasingly deployed in high-stakes settings ranging from medical diagnosis~\citep{gaber2025evaluating} to legal research and case analysis~\citep{dehghani2025large}. 
Yet even the most sophisticated models available today still routinely hallucinate false claims~\citep{Fan2026HalluHardAH}.
This has motivated a growing body of work on quantifying and reducing hallucinations, including methods that train LLMs to predict their own confidence~\citep{damani2025beyond},
methods that produce calibrated predictions of factuality~\citep{detommasomulticalibration}, or methods that generate sets of outputs containing a factual response with high probability \citep{quach2024conformal}. 

One particularly promising recent direction has been to leverage  \emph{conformal prediction}
\citep{mohri:hashimoto,cherian2024large,rubin-toles2025conformal},
a statistical wrapper yielding assumption-lean guarantees.
This line of work seeks to control the expected value of a loss function $\ell : X \times Y \to \mathbb{R}_{\geq0}$,
such as a measure of an output $y$'s factuality with respect to a prompt $x$.
To control this quantity, conformal methods split the output into distinct claims, and leverage a heuristic score function, which serves as an (imperfect) proxy of the loss. By observing the relationship between the scores and the true losses on a held-out calibration set, these methods then calibrate a \emph{post-hoc filtering} mechanism that provably controls the risk.
Formally, given a calibration set $\dcal = \{(x_i, y_i, \ss_i, \ell_i)\}_{i=1}^n$ of $n$ samples from the LM alongside their claim-level scores $s_{ij}$ and observed losses $\ell_i = \ell(x_i, y_i)$, they construct a mapping $F$ that filters out
low-confidence claims so that the resulting $F(\ytest)$ satisfies $\E[\ell(\xtest, F(\ytest)] \leq \alpha$, for some user-defined risk tolerance $\alpha \in (0,1)$.

While post-hoc surgical altering of outputs makes sense in situations where the user is satisfied with a partially filtered generation, it may lead to samples that have low downstream utility in tasks where claims have strong interdependencies. In other words, while post-hoc filtering guarantees a bound on the expected number of responses containing a false claim, it often renders many of the model's responses \emph{incomplete} or simply \emph{unhelpful} to the user.

To produce samples that are fluent, coherent and satisfying to the user, such methods therefore typically repair the filtered output using either a language model \citep{mohri:hashimoto}, thus reintroducing the risk of fabricating claims, or change the filtering mechanism symbolically through complicated dependency-graph schemes \citep{rubin-toles2025conformal}. However, even with some degree of fluency restored, these methods cannot recover from (over-)agressively filtered responses.

Against this backdrop, in this paper, we ask:
\emph{Can we design a \textbf{risk-controlling} calibration routine that respects the underlying generative process and moves probability mass  towards \textbf{useful} responses whenever the model has the capacity to solve the task at hand?}

We answer this question affirmatively by shifting the target of calibration from a post-hoc filter to the sampling distribution itself. Our key idea is that risk control should reweight the LM's distribution towards completions that are more reliable and more useful, rather than altering the completions in a post-hoc manner. We  reframe conformal language modeling as posterior sampling from the LM, conditioned on an event that corresponds to the response falling within a high-confidence region of the output space. That is, instead of calibrating a post-hoc filtering mechanism, we directly calibrate a sampling distribution  $Q(\cdot | \xtest, \dcal)$ such that
\begin{equation}
\label{eq:introduction:riskcontrol:objective}
    \E_{\dcal, \xtest,\ytest \sim Q(\cdot \mid \xtest, \dcal)}[\ell(\xtest,\ytest)]  \leq \alpha.
\end{equation}

\subsection{Roadmap and Contributions}
\xhdr{Modeling} In Sec.~\ref{section:preliminaries}, we formulate conformal LM factuality control as \emph{posterior sampling} from a calibrated family, indexed by a one-dimensional parameter $\tau$. We define a family of potentials which threshold completions by their cumulative claim-level factuality scores. This amounts to conditioning the LM distribution on high-confidence regions. While, as we show, this is cleanly motivated by a probabilistic formulation, our method's safety guarantees will be agnostic to the choice of score function.

\xhdr{Calibration} In Sec.~\ref{sec:calibration}, we develop an \emph{off-policy} calibration procedure for selecting the posterior threshold $\tau$, \emph{without requiring samples from each candidate posterior}. We face two major challenges: firstly, the posterior's conditional normalizing constant $Z^\tau(x)$ now needs to be estimated from samples, which is ill-conditioned (or even ill-defined) when the model has low (or no) mass on the high-confidence region. Secondly, estimating a plug-in of the importance reweighted loss naively yields a non-monotone objective, making conformal guarantees hard to obtain. We solve the former issue by modeling the calibration objective as an \emph{empirical distribution posterior} that is extended by an explicit abstention path, effectively modeling the prior as a mixture distribution. We solve the second issue by explicitly monotonizing the objective, and show experimentally that this has a negligible impact on conservativeness.
Finally, we connect the empirical posterior distribution to the true population posterior. Using a Glivenko-Cantelli argument, we show that the empirical posterior risk converges to the population posterior risk as the number of calibration samples grows.

\xhdr{Experiments} In Sec.~\ref{sec:experiments}, we evaluate our method on two case studies: open-ended biography generation~\citep{Min2023FActScoreFA} and math reasoning~\citep{Hendrycks2021MeasuringMP}. Our method tracks the desired factuality level while significantly \emph{improving downstream utility} relative to a post-hoc conformal filtering baseline, especially in settings with strong claim interdependencies.

\section{Preliminaries and Modeling}
\label{section:preliminaries}

\newcommand{\lmprior}{p_{\mathrm{LLM}}}
\newcommand{\llm}{\mathrm{LLM}}
\subsection{The Posterior Family}
\label{section:posterior}
We model an LLM with vocabulary $\calV$ and prompt $x$ as a distribution $\lmprior(y \sep x) \in \Delta(\calV^*)$ over completions, where $\calV^*$ is the Kleene closure of $\calV$ (up to some max length $L$, omitted for brevity). We call $\lmprior$ the LM \emph{prior} and extend it with a potential function $\Phi(y \mid x)$ to obtain a \emph{posterior}
\begin{equation}
\label{eq:posterior:original}
    p_{\mathrm{LLM}}^{\Phi}(y \sep x) \propto \lmprior(y \sep x) \Phi(y\mid x).
\end{equation}
Our goal is to pick $\Phi$ so that substituting $p_{\mathrm{LLM}}^{\Phi}(\cdot | x)$ for $Q(\cdot | x)$ in~\Cref{eq:introduction:riskcontrol:objective} controls risk at the target level. To do so, we make $\Phi(\cdot | x)$ depend on a heuristic notion of how ``risky'' $y$ is, and control the perceived risk by setting a one-dimensional parameter $\tau$, yielding a posterior $p_{\llm}^{\Phi^\tau}(\cdot | x) = p_{\llm}^\tau(\cdot | x)$ defined via the parametrized potentials $\Phi^\tau$. The RHS of Equation~\eqref{eq:posterior:original} implicitly depends on the \emph{normalization constant} $Z^{\tau} (x) = \int_{y \in \calV^*} \Phi^\tau(y | x) \d p_{\llm}(y | x)$, which --- as we will see below --- corresponds to the probability mass that survives the threshold $\tau$ for prompt $x$. We next turn to the question of how to instantiate $\Phi^\tau$ for a given $\tau$.

\subsection{From Heuristic Scores to Potentials}
\label{sec:motivational}
To build a family of potential functions that allow us to control the risk, we draw inspiration from prior work on conformal prediction
and leverage a heuristic scoring function.
We design this scoring function so that it gives an estimated confidence for $y$'s most recent claim $c_i$ being correct conditional on the prompt $x$ and all preceding claims $\cc_{< i} = (c_1,\ldots,c_{i-1})$ being assumed to be correct.\footnote{In this paper, we will routinely use the notation $\uu_{\leq i}$ and $\uu_{< i}$ to denote prefix vectors of sequences $(u_i)_i$.}
Formally, that is, the scoring function is of the form $
s: \mathcal{V}^* \times \mathcal{V}^* \times (\mathcal{V}^*)^* \rightarrow \mathcal{S},$
where $\mathcal S \subset \mathbb R$ is chosen to be $(0,1]$ w.l.o.g.\ 
Using $\mathbf{c}_{<i}$ to denote the prior claims whenever clear from context, we will abuse probabilistic conditioning notation and write
$s(c_i,x,(c_1,\ldots,c_{i-1}))$ as
$$
s_i := s(c_i \sep x,\mathbf{c}_{<i}) := s(c_i \sep x,c_1,\ldots,c_{i-1}).
$$
Intuitively, the ``oracle'' score $$s(c_i  \sep x, \cc_{<i}) = \indicator{c_i \text{ is factual given } \cc_{< i} \text{ and prompt } x }$$ would immediately give us an $\alpha=0$ level risk-controlling potential by letting $\Phi(y \sep x) = \indicator{\prod_{i} s_i = 1}$, ensuring that all samples from $p^\Phi_{\llm}$ contain only true claims. In practice, such an oracle cannot be obtained.
The purpose of this paper is to show that with a heuristic choice of $s$,
we can, however, still control the risk by appropriately picking a threshold $\tau$ such that
\begin{equation}
\label{eq:phitau}
\Phi^\tau(y \sep x) = \indicator{\prod_{i=1}^K s(c_i \sep x, \mathbf{c}_{<i}) \geq \tau},
\end{equation}
where $K$ is the (random) number of claims in the output. We will sometimes write $S(y\sep x) = \prod_{i \in [K]} s(c_i \sep x,\cc_{<i})$ to denote completion-level scores. We build intuition for this definition by recalling the chain rule of probabilities. Assume that we are in an idealized scenario where the scores $s_i$ act as correctness forecasts satisfying a weaker-than-oracle (but still unrealistic) notion of conditional calibration. 
Informally, assume (see also a detailed derivation in Appendix~\ref{appendix:motivational})that  
\begin{equation}
	\P[\cc_{i} \text{ is factual}  \sep \ss_{\leq i}; \text{previous claims }\cc_{< i} \text{ are all factual}] = s_i . \label{eq:motivationalassumption:informal} 
\end{equation} 

That is,
the score is a calibrated forecast of the correctness of the present claim given that all preceding claims were factual. Under this assumption and additional technicalities outlined in Appendix~\ref{appendix:motivational}, where the statement can be found in full detail, one can show the following observation.
\begin{observation}[Idealized validity; informal]\label{prop:idealized-informal}
Under an idealized, stepwise-calibrated\footnote{By ``stepwise-calibrated'', we mean that each claim-level score is calibrated as a forecast of the current claim's correctness, conditional on the prompt, the previous scores, and the event that all strictly earlier claims were correct.} scoring function $s(\cdot)$, sampling $y \sim 
p_{\llm}^{1-\alpha}(\cdot \sep x)$ induced by $\Phi^{1-\alpha}(y \sep x)$ satisfies $\E[\ell(x,y)] \leq \alpha$.
\end{observation}
Intuitively, if the scores were stepwise-calibrated forecasts of conditional correctness given the preceding prefix, a martingale argument (Appendix~\ref{appendix:motivational}) would yield $\prod_{i\leq K} s_i = \P(y \text{ correct} \sep x, \ss_{\leq K})$.
The product of scores then acts as a sequential probability-of-correctness estimate, and the choice $\tau = 1-\alpha$ directly (conservatively) controls risk at level $\alpha$.

In practice, heuristic scores are far from calibrated in this sense, and it is generally impossible to achieve calibrated forecasts that are this strong \citep{van2016calibration, hebert2018multicalibration}. We must therefore choose $\tau$ from data via a procedure that makes \emph{essentially no assumptions on the scorer}; this is the main theoretical subject of this paper, which we return to in \Cref{sec:calibration}.

\subsection{Calibrated Posterior Sampling on a Mixture Prior}
\label{sec:mixture}
Before we turn to the question of calibrating the threshold $\tau$, we note that our discussion so far relies on the assumption that the posterior is well-defined for all $x$ and $\tau$.

Formally, this amounts to requiring $Z^{\tau}_{\llm}(x) = \int_{y \in \calV^*} \Phi^\tau(y \sep x) \d p_{\llm}(y\sep x) > 0$ for all $x \in \calX$ and $\tau \in [0,1]$. Since this would require the ability to generate samples that pass the potential's thresholding mechanism by mapping to thresholds that are arbitrarily close to 1, this is too stringent of an assumption on the scoring mechanism. Post-hoc methods such as that of \citet{mohri:hashimoto} sidestep this issue by treating an all-claims-dropped generation (i.e., outputting the empty set of claims) as an implicit abstention that is always factual. In our setting, $p^\tau_{\llm}$ inherits its support from $\lmprior$, which may not include an explicit abstention path. The language model prior $\lmprior$ may have no practical amounts of mass on any score-passing outputs at all; strict thresholds thus leave no direct route to safe output, prohibiting calibration. To get around this issue, we therefore calibrate not against $\lmprior$ itself but against the \emph{mixture prior}
\begin{equation}
\label{eq:mixture-prior}
    p_\beta(y \sep x)  = (1 - \beta)\, p_{\mathrm{LLM}}(y \sep x) + \beta\, \delta_{\texttt{abstain}}(y),
\end{equation}
where $\delta_{\texttt{abstain}}$ is a point mass on a designated abstention response, $\beta > 0$ is a small pre-specified mass, and we require that the scorer assigns $S(\texttt{abstain} \sep x) = 1$ for all $x \in \calX$.
The resulting posterior then retains a factual \emph{escape valve} at cutoff $\tau$; abstaining is always factual.  Analogously to \eqref{eq:posterior:original}, we define the mixture posterior 
$
    p_\beta^\tau(y \sep x) \propto p_\beta(y\sep x) \Phi^\tau(y \sep x)
$.
Conveniently for what follows, this also makes the conditional normalizing constant well behaved, i.e., $Z^\tau_\beta(x) = \int_{y \in \calV^*} \Phi^\tau(y \sep x) \d p_\beta(y\sep x) \geq \beta$ for all $\tau$. We often drop the subscript $\beta$ and just write $p := p_\beta$ and $p^\tau = p^\tau_\beta$, and similarly $Z^\tau = Z^\tau_\beta$. 
Note that $p^0 = p^0_\beta = p_\beta$ is simply the prior, since thresholding has no effect at $\tau=0$ when $\mathcal{S} \subseteq [0,1].$ 

\subsection{Instantiating the Scoring Function}
\label{sec:instantiating-scoring}
While the validity of our method is agnostic to the accuracy of the scorer, yielding risk-controlling guarantees under very minimal assumptions, the scorer does have an impact on how stringently the system filters out generations
(see Appendix~\ref{appendix:scorer-elicit-vs-logprob} for a scorer ablation).
We follow prior work and elicit scores by asking an external LLM to rate its confidence in each claim on a scale from $0$ to $1$~\citep{mohri:hashimoto}.
Empirically, this does not yield continuously-distributed scores, instead concentrating on a small, discrete set. Following \citet{mohri:hashimoto}, in order to calibrate to arbitrary targets of factuality $1-\alpha$, we therefore add a small random perturbation. For an elicited score $\tilde s_i$, we set the final score as $s_i \sim U[\max(0, \tilde s_i - \gamma), \min(1, \tilde s_i + \gamma)]$, where $\gamma = 0.01$ is a small constant, the exact choice of which has no significant downstream effects, so long as it is on the order of the granularity of the elicited scores\footnote{LLM scorers tend to return scores that are capped at one or two decimal points.}.

While having flexibility on the choice of scorer is convenient, it also presents a methodological challenge. Since \(p^\tau(\cdot \mid x)\) depends on both \(s\) and \(\tau\), calibration via fixed-sequence testing~\citep{Angelopoulos2021LearnTT} would require specifying a sensible grid of thresholds $\tau$ to consider in advance, and drawing a new calibration set for each \((s,\tau)\) pair --- a prohibitively expensive process that would require to be repeated whenever $s$ is modified.
Even setting computational burden aside, how to pick such a grid is not clear a priori, especially because the spread of the scores is unknown and could be concentrated within some small range, making it difficult to select a sensible discretization a-priori.
Instead,
we will aim to develop an \emph{off-policy} calibration routine which only requires samples from the (mixture) prior,
eliminating the need for expensive on-policy data collection altogether.

\section{Methodology: Calibrating the Threshold $\tau$}

\label{sec:calibration}

Having defined the potential family $\Phi^\tau$ and the mixture prior $p$, we turn to the central question of the paper: how to pick $\tau$ so that posterior sampling controls risk at level $\alpha$, using \emph{only} samples from the base model. We will have access to a calibration set of prompts $\mathcal{C} = \setof{x_1,\ldots,x_n}$ from $\covdist$, and assume that $x_1,\ldots,x_n,\xtest$ are exchangeable. We present our approach in 3 steps: (1) derivation of an idealized calibration rule under known normalization constant; (2) construction of a tractable empirical approximation; and (3) proof that this yields risk control for the resulting posterior sampler(s).

\xhdr{Known Normalization Constant} Assume that we have access to an oracle for the \emph{conditional} normalization constant $Z^\tau(x) = \int_{y \in \calV^*} \Phi^\tau(y \sep x) \d p(y\sep x)$. If that was the case, then we could estimate the posterior risk using importance reweighting. To see this, note first that from Section~\ref{section:posterior}: $\frac{p^\tau(y \sep x)}{p(y\sep x)} = \frac{\Phi^\tau(y \sep x)}{Z^\tau(x)}$. So, instead of sampling directly from the posterior $p^\tau$, we can i.i.d.\ sample outputs $y_i$ from the mixture prior $p(\cdot \sep x_i)$, keep only samples whose score exceeds $\tau$, and then re-weigh their losses to account for the fact that they were sampled from the prior rather than the posterior.
To obtain a finite-sample guarantee, we add a conservative correction term $1 /(\beta (n+1))$ for the unseen test point and conservatively monotonize the objective by lower bounding the likelihood ratio's denominator. This gives an \emph{idealized} calibration rule:
\newcommand{\idealizedrule}{\hat{\tau}_{\mathrm{id}}}
\begin{align}
\idealizedrule = \inf\bigg\{\tau \, \bigg| \, \frac{1}{n+1} \sum_{i=1}^{n} \frac{\Phi^\tau(y_i \sep x_i)}{Z^{S(y_i\sep x_i)}(x_i)} \ell(x_i,y_i)  +  \frac{1}{ \beta(n+1)}  \leq \alpha \bigg\}. \label{eq:tauselection:idealizedcase}
\end{align}
We briefly comment on this calibration rule. 
Note that since $\tau \mapsto Z^{\tau}(x_i)$ is a \emph{monotone} non-increasing\footnote{This is because as $\tau$ increases, the probability mass of completions that pass $\tau$ can only decrease.} map and by the thresholding nature of $\Phi^\tau$, each term in the sum is an upper bound on $\frac{\Phi^\tau(y_i | x_i)}{Z^\tau(x_i)}\ell(x_i,y_i)$. Fixing the denominator to not depend on $\tau$ has an additional crucial property: the ratio is now monotone in $\tau$. Therefore, the LHS of the constraint in \eqref{eq:tauselection:idealizedcase} inherits this \emph{monotonicity}, yielding our first risk control guarantee. 

\begin{proposition}
\label{prop:idealized}
    If $x_1,\ldots,x_{n}, \xtest$ are exchangeable, samples from $\covdist$ and we have access to $Z^\tau(x_i)$ for all $\tau \in [0,1]$, then computing $\idealizedrule$ according to~\Cref{eq:tauselection:idealizedcase} and sampling the test point from the \emph{true} posterior $\ytest \sim p^{\idealizedrule}(\cdot \sep \xtest)$ satisfies $
    \E[\ell(\xtest, \ytest) ] \leq \alpha,$ 
    where the randomness is over the draw of the calibration prompts, the calibration responses, and the draw from the test-time posterior.
\end{proposition}
The proof is deferred to Appendix~\ref{appendix:proof:idealized}.

\xhdr{Unknown $Z^\tau(x)$}
In practice, computing $Z^\tau(x)$ is intractable\footnote{Indeed, computing it requires summing/integrating over all possible completions $y$ that the LM could generate.}, and the rule \eqref{eq:tauselection:idealizedcase} cannot be evaluated.
A natural idea is to estimate, at each calibration prompt, the conditional posterior risk $\E_{y \sim p^\tau(\cdot \sep x_i)}[\ell(x_i, y)]$ from base-model samples.
To this end, we draw $M$ i.i.d.\ responses $y_{ij} \sim \lmprior(\cdot \sep x_i)$ (aka \emph{``particles''}) for each $i \in [n]$ and form the empirical (mixture) prior
\begin{equation}
    \label{eq:empirical-prior}
    \hat p(y \sep x_i, T_i) = (1-\beta)\,\frac{1}{M}\sum_{j \in [M]} \delta_{y_{ij}}(y) + \beta\, \delta_{\texttt{abstain}}(y),
\end{equation}
where we use the shorthand $T_i = (y_{i1}, \ldots, y_{iM})$ and the abstention atom is attached analytically with mass $\beta$ rather than sampled. This stratification matches the abstention weight to that of the population mixture prior $p_\beta$ but, crucially, \emph{ensures that there is deterministic abstention mass,  making the plug-in estimator well-defined}.
The induced \emph{empirical posterior} is
\begin{align}
    \hat p^\tau(y \sep x_i, T_i) = \frac{\Phi^\tau(y \sep x_i)\, \hat p(y \sep x_i, T_i)}{\hat{Z}^\tau(x_i, T_i)}, \quad
    \hat Z^\tau(x_i, T_i) = \beta + (1-\beta)\,\frac{1}{M}\sum_{j \in [M]} \Phi^\tau(y_{ij} \sep x_i).     \label{eq:empirical-posterior}
\end{align}
The explicit abstention path implies $\hat Z^\tau(x_i, T_i) \geq \beta$ which makes $\hat p^\tau$ well-defined for all $\tau \in [0,1]$.

\xhdr{Calibrating against the empirical posterior}
 The definition of $\hat{p}$ signifies that we are essentially stratifying the sampling from both modes of the mixture prior. This will naturally lead to a calibration routine that does \emph{not} control the risk under the population posterior $p^\tau$, but instead under the empirical posterior $\hat p^\tau$. To this end, given $T_i$, we consider the conditional risk under $\hat p^\tau$, namely $\E_{y \sim \hat p^\tau(\cdot \sep x_i, T_i)}[\ell(x_i, y)]$, and write the shorthand 
\begin{equation}
    \label{eq:empirical-conditional-risk}
    \hat H_i(\tau) = \frac{(1-\beta)\frac{1}{M}\sum_{j \in [M]} \Phi^\tau(y_{ij} \sep x_i)\,\ell(x_i, y_{ij})}{\beta + (1-\beta)\frac{1}{M}\sum_{j \in [M]} \Phi^\tau(y_{ij} \sep x_i)},
\end{equation}
because the abstention atom contributes $\ell(x, \texttt{abstain}) = 0$ to the numerator and can be omitted.
$\hat H_i(\tau)$ is a deterministic function of the calibration data alone, with no further randomness over a sampler: calibration simply inspects the support of $\hat p^\tau$, it does not draw from the corresponding empirical distribution.

As a step function of $\tau$\footnote{To see this, note that the only way that $\tau$ affects $\hat{H}_i(\tau)$ is through the score indicators $\Phi^{\tau}(y_{ij} | x_i) = \indicator{ S(y_{ij} | x_i) \geq \tau}$. As such, $\hat{H}_i(\tau)$ only changes when $\tau$ crosses one of the sampled scores $S(y_{ij} | x_i)$.}, $\hat H_i$ need not be monotone: pruning a correct particle at a slightly higher threshold while leaving incorrect particles to survive can transiently raise the average loss among the accepted particles. We restore monotonicity sample-pathwise via the monotone upper envelope:
\begin{equation}
    \label{eq:monotonized}
    \bar H_i(\tau) = \sup_{\tau' \geq \tau} \hat H_i(\tau'),
\end{equation}
a non-increasing pointwise upper bound on $\hat H_i$. Monotonization is the lever that makes conformal-risk-control style arguments work. Empirically (see Sec.~\ref{sec:experiments}), $\hat H_i$ is close to monotone for moderate $M$, and $\bar H_i$ introduces little slack.

The full procedure is given in \Cref{algo:crc}. We can give the following conformal guarantee, which approaches a guarantee on a sample from the population posterior as $M$ grows (c.f. Corollary~\ref{corollary:limsup}).

\begin{algorithm}[t!]
\caption{Off-policy Conformal Risk Control with Path Monotonization}
\label{algo:crc}
\begin{algorithmic}[1]
\Require Calibration prompts $x_1,\ldots,x_n$; particle budget $M$; abstention mass $\beta$; target level $\alpha$
\For{$i = 1,\ldots,n$}
    \For{$j = 1,\ldots,M$}
        \State Sample $y_{ij} \sim \lmprior(\cdot \sep x_i)$, score it to obtain $\Phi^\tau(y_{ij} \sep x_i)$, and annotate $\ell(x_i, y_{ij})$
    \EndFor
\EndFor
\State Form $\hat H_i(\tau)$ via \eqref{eq:empirical-conditional-risk} and $\bar H_i(\tau)$ via \eqref{eq:monotonized}
\State \Return $\displaystyle
\hat\tau = \inf\!\setof{\tau \bigg| \frac{1}{n+1}\sum_{i=1}^{n} \bar H_i(\tau) + \frac{1}{n+1} \leq \alpha}.$
\end{algorithmic}
\end{algorithm}

\begin{theorem}[Finite-sample risk control for the empirical posterior]
\label{thm:main}
Let prompts $x_1, \ldots, x_n, \xtest$ be exchangeable, and assume that conditionally, all calibration and test particles are drawn independently with $y_{ij} \sim \lmprior(\cdot | x_i)$ and $(\ytest)_j \sim \lmprior(\cdot | \xtest)$. Let $\hat\tau$ be the output of \Cref{algo:crc}. If $\ytest \sim \hat p^{\hat\tau}(\cdot | \xtest, T_{\mathrm{test}})$, then $\E[\ell(\xtest, \ytest)] \leq \alpha$, 
where the expectation is over the calibration prompts and particles, test prompt and particles, and the draw from the test-time posterior.
\end{theorem}

Given base-model samples, per-claim scores, and factuality annotations on the calibration set, \Cref{algo:crc} returns a threshold $\hat\tau$ which controls the risk of the empirical-posterior sampler.
The proof (in Sec.~\ref{appendix:proof:theorem32}), follows a standard conformal-risk-control template: $\bar H_i$ is a monotone envelope of a plug-in estimator for the empirical-posterior risk, exchangeability of the calibration tuples bounds an ``oracle'' threshold $\tau^\star$ at level $\alpha$, and monotonicity transfers this bound from $\tau^\star$ to $\hat\tau$.

\xhdr{From empirical to population posterior}
Calibrating against $\hat p^{\hat\tau}$ raises a natural question: does the guarantee transfer to the population posterior $p^{\hat\tau}$? Because $\hat\tau$ is data-dependent, pointwise convergence at fixed $\tau$ is not enough; we need the convergence to be uniform in $\tau$. Luckily, the threshold-indexed family $\{\Phi^\tau(\cdot \sep x) : \tau \in [0,1]\}$ is a one-dimensional threshold class and hence Glivenko-Cantelli, giving us the following result.

\begin{lemma}[Uniform asymptotic equivalence of conditional risks]
\label{lemma:glivenko}
Fix abstention mass $\beta > 0$, and let $T_M = (y_1, \ldots, y_M)$ with $y_j \stackrel{\text{i.i.d.}}{\sim} \lmprior(\cdot \sep x)$. Then for $\covdist$-a.e.\ $x$,
\begin{equation*}
    \sup_{\tau \in [0,1]} \abs{\hat \E_{y \sim \hat p^\tau(\cdot \sep x, T_M)}[\ell(x, y)] - \E_{y \sim p^\tau(\cdot \sep x)}[\ell(x, y)]}\stackrel{\text{a.s.}}{\rightarrow} 0.
\end{equation*}
\end{lemma}
Note that this is an almost surely statement about the draw of the calibration particles $T_M$, as $M \rightarrow \infty$, uniformly over all $\tau \in [0,1]$. Combined with \Cref{thm:main}, this gives us the following corollary.
\begin{corollary}
\label{corollary:limsup}
As the number of calibration particles grows, for $\ytest \sim p^{\hat\tau_M}(\cdot \sep \xtest)$ we have $\limsup_{M \to \infty} \E[\ell(\xtest, \ytest)] \leq \alpha$.
\end{corollary}
In other words, the calibration guarantee transfers asymptotically to any deployment sampler that targets the population posterior. This means that we achieved the goal we set out to, modulo error that shrinks as $M$ grows. We also point out that we need not necessarily enter the asymptotic regime for the finite-particle, stratified way of constructing the empirical posterior to make sense. Different ways to sample from the posterior are possible that still exactly or approximately inherit the finite sample guarantee in Theorem~\ref{thm:main}. For instance, Sequential Monte Carlo steering \citep{lew2023sequential,zhao2024probabilistic,loula2025syntactic} fixes the initial number of particles, in which case our method amounts to extending the proposal distribution with an appropriately weighted abstention particle; see \Cref{sec:related,sec:conclusion} for further discussion on efficient sampling and particle filtering methods.

\section{Experiments}
\label{sec:experiments}
We now empirically evaluate our framework's validity (that is, whether it controls the level of risk at test-time) and utility (whether the resulting generations are useful).
We focus on two distinct case studies: one within the hallucination-prone task of open-ended biography generation, and one within mathematical reasoning where inter-claim dependencies are stronger.

\xhdr{Datasets and models}
For the biography generation,
we use 682 entities from the FActScore dataset \citep{Min2023FActScoreFA}; 
for mathematical reasoning, we use the standard test set of 500 problem instances from MATH \citep{Hendrycks2021MeasuringMP}.
We average over 10 random splits of the data, with 50 instances used as the test set and the remaining used for calibration; results are reported as the mean over the splits, with shaded areas showing approximate 95\% CIs using the studentized Wald interval.
For generation we use \texttt{Llama-3.3-70B-Instruct}\footnote{To reduce cost, we use the (4-bit weights, 16-bit activations)-quantized variant provided by RedHatAI on \href{https://huggingface.co/RedHatAI/Llama-3.3-70B-Instruct-quantized.w4a16}{\texttt{HuggingFace}}.} \citep{grattafiori2024llama}, a model with intermediate capacity on these tasks.
We serve the model on a single NVIDIA RTX A6000 with 48GB of VRAM, taking $\sim$30s per prompt to generate answers.
To elicit \emph{scores}, we use OpenAI's GPT-4o, a larger and more capable (but still unreliable) estimator of truth; App.~\ref{appendix:scorer-elicit-vs-logprob} compares against a raw token log-probability scorer and the cheaper GPT-4o-mini.
Finally, to ensure that the per-claim \emph{labels} are as accurate as possible, we label the data using GPT-5.4; for FActScore, we include a copy of the entity's Wikipedia article in the prompt.

\xhdr{Base generations and mixture prior setup}
We sample 20 base generations from $\lmprior$ for each instance, using a moderate temperature of 0.8 (prompts in Appendix~\ref{appendix:generation-prompts}).
We fix $\beta = 0.1$ for both datasets
and report plug-in estimates of the metrics using the empirical posterior over each instance's 20 generations.
For example, to compute the empirical risk over a test set $D_\text{test} = \{x_i\}_{i=1}^{N}$, we average \eqref{eq:empirical-conditional-risk} over the test instances $x_i$.
We note that for a fixed and finite number of calibration samples $M$, at very small values of $\beta$, the asymptotic guarantees of  Corollary~\ref{corollary:limsup} may not translate to the population posterior $p^\tau$; for a thorough investigation, see Appendix~\ref{appendix:rs-ablation}.

\xhdr{Baseline}
We compare against \citet{mohri:hashimoto}, the closest prior work.
We re-implement \citet{mohri:hashimoto}'s algorithm so that it does not require the use of a separate model to decompose generations into atomic claims; instead, we encourage the generator to keep clauses brief and split on sentences.
We do so to illustrate a tight integration of each method into the generation loop without any reliance on additional restructuring steps, and to enable a direct comparison to our method.
For brevity, we refer to this baseline as \texttt{MH}.

\begin{wrapfigure}{r}{0.35\linewidth}
    \vspace{-3.6em}
    \centering
    \includegraphics[width=\linewidth]{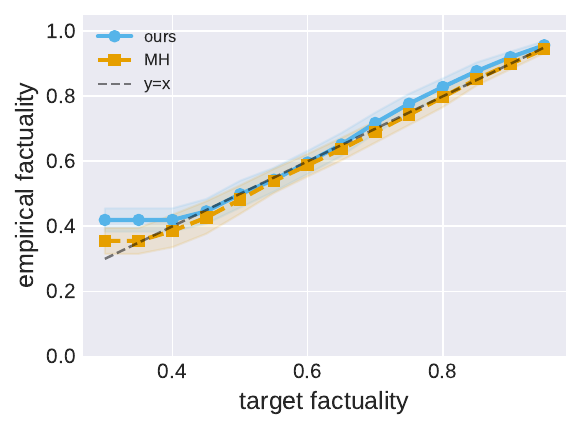}
    \caption{Empirical vs.\ target factuality on FActScore; ours (\colorours) and \texttt{MH} (\colormh). Both attain valid coverage.}
    \label{fig:calibration}
    \vspace{-2em}
\end{wrapfigure}

\subsection{Validity}
\label{sec:exp-calibration}

\Cref{fig:calibration} shows the empirical risk obtained versus the desired target $1-\alpha$ on FActScore (results on MATH in App.~\ref{appendix:math-calibration}).
As expected from \cref{thm:main}, our method achieves valid coverage at all levels. Coverage is also empirically tight; monotonization of the objective does not introduce excessive slack.
Note that low targets ($\approx 0.4$ for FActScore, $\approx 0.5$ for MATH) fall below the base model's capacity, so both methods reduce to sampling directly from the prior and therefore flatline. The remaining gap between the two is due to $\beta$. As our method samples from the mixture prior $\hat p$ instead of $\prior$, it is slightly more conservative at targets below the model's base capacity.

\subsection{Utility}
\label{sec:exp-utility-efficiency}

We use LLM-judge scoring as our primary measure of the downstream utility of the samples we obtain.
For FActScore we construct a Likert-style questionnaire along three axes (\emph{completeness}, \emph{fluency}, and \emph{helpfulness}; rubric in Appendix~\ref{appendix:judge-likert}) and ask GPT-5.4 to rate its level of agreement between 1-5 for each.
Abstentions are passed to the judge as a pre-formatted response of the form ``\emph{I don't know who or what \{entity\} is.}``
For MATH, we simply ask GPT-5.4 whether the proof is \emph{complete}, that is, whether it leads to a final answer and that final answer matches the reference (prompt in Appendix~\ref{appendix:judge-completeness}).
Abstentions are naturally scored as incomplete.
To reduce annotation cost, we subsample a single generation\footnote{For our method, from the posterior. For \texttt{MH}, from the base model with subsequent filtering.} for each prompt $x_i$ for this experiment.

In addition to these metrics, we also compute a series of proxy metrics on all the outputs $y_{ij}$: (i-ii) the average number of claims per output, both with abstentions counting as zero claims and conditioned on non-abstention; (iii) the total fraction of emitted claims that are true, aggregated over all prompts; and (iv) the fraction of non-abstaining outputs that are \emph{entirely} factual, with no false claims at all.

\xhdr{Downstream results}
\Cref{fig:utility-downstream} reports the FActScore Likert and MATH completeness scores, both unconditionally (with abstentions scored as described above) and conditioned on obtaining a non-abstaining sample.
\emph{FActScore:}
Non-abstaining generations from our method degrade less sharply than those from \texttt{MH} as the target factuality increases; the per-axis breakdown in \Cref{fig:utility-likert-axes} (Appendix~\ref{appendix:likert-axes}) reveals that this is predominantly driven by less degradation in completeness and helpfulness, rather than fluency.
In the unconditional average, which conflates response \emph{quality} with response \emph{frequency}, our method underperforms \texttt{MH}.
This indicates that the LM judge prefers a highly edited response to refusals in this setting; how such a refusal should be valued relative to a riskier-but-partially-informative answer ultimately depends on the context in which the system is deployed, and this may not reflect the preference of actual users.
\emph{MATH:} With a more objective notion of utility, the gains from our method become clearer: while the mixture prior initially increases the probability of sampling an abstaining (and thus incomplete) response, \texttt{MH} begins to generate incomplete answers at very high rates for targets $\geq 0.6$.
Our method not only \emph{degrades less}, but in the instances where it does not abstain, the rate of completeness even \emph{increases} with the target factuality.
These findings thus suggest that our method produces useful generations even at high target levels of factuality.

\begin{figure*}[t]
    \centering
    \begin{subfigure}[t]{0.24\textwidth}
        \includegraphics[width=\linewidth]{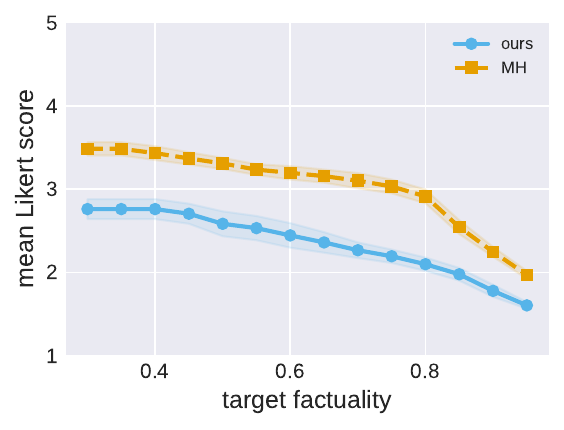}
        \caption{FActScore Likert.}
        \label{fig:fs-likert-uncond}
    \end{subfigure}\hfill
    \begin{subfigure}[t]{0.24\textwidth}
        \includegraphics[width=\linewidth]{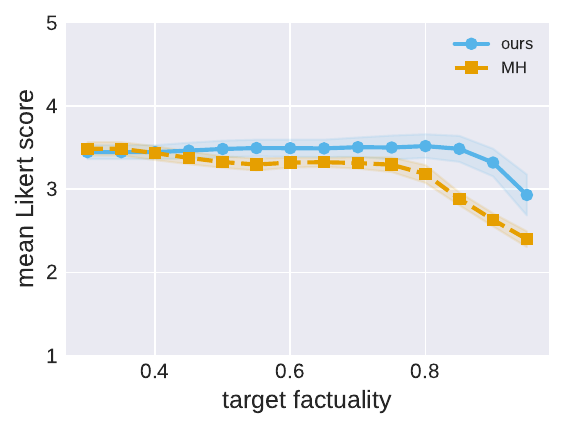}
        \caption{Likert\,$\mid$\,non-abstain.}
        \label{fig:fs-likert-cond}
    \end{subfigure}\hfill
    \begin{subfigure}[t]{0.24\textwidth}
        \includegraphics[width=\linewidth]{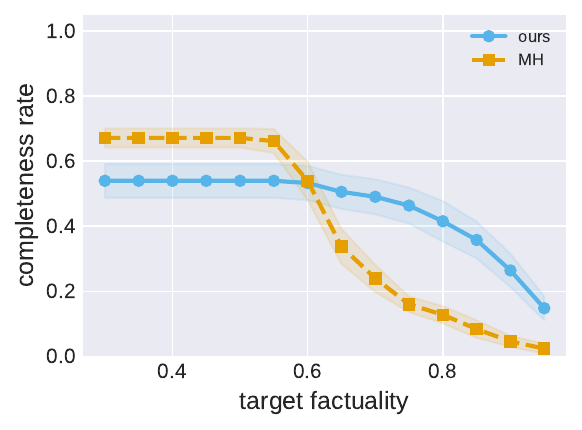}
        \caption{MATH completeness.}
        \label{fig:math-comp-uncond}
    \end{subfigure}\hfill
    \begin{subfigure}[t]{0.24\textwidth}
        \includegraphics[width=\linewidth]{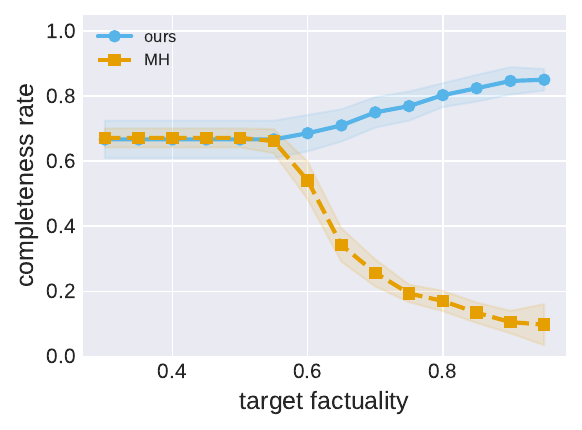}
        \caption{Complete\,$\mid$\,non-abstain.}
        \label{fig:math-comp-cond}
    \end{subfigure}
    \caption{LLM-judge utility for ours (\colorours) vs.\ \texttt{MH} (\colormh) as target factuality varies. Unconditional panels (a,~c) include abstentions (judged on FActScore, marked incomplete for MATH); conditional panels (b,~d) show results for both methods conditioned on the output not being an abstention.}
    \label{fig:utility-downstream}
    \vspace{-2em}
\end{figure*}

\xhdr{Proxy results} 
\Cref{fig:utility-metrics} shows the claim-level proxy metrics, with FActScore on the top row and MATH on the bottom.
\emph{First column}: At intermediate targets, our method emits roughly as many claims per instance as \texttt{MH}.
The mixture prior places $\beta$ of the mass on zero-claim abstentions, deflating our method's average at low targets; this reverses at high targets, where \texttt{MH} increasingly produces heavily filtered generations.
\emph{Second column:}
Restricting our attention to non-abstaining responses, our method emits substantially more claims on both datasets.
Intuitively, \texttt{MH} distributes rejected claims evenly across instances, while our method concentrates filtering on hard instances and leaves easier prompts with full-length generations.
\emph{Third column}:
The fraction of emitted claims which are correct is similar for both methods, increasing (as expected) with the target level.
\emph{Fourth column}:
The one axis on which our method may lose ground is the fraction of non-abstaining outputs that are \emph{entirely} factual.
On FActScore in particular, we observe a slight decrease in the chance that \emph{conditioned on not being an abstention}, any given sample is completely factual.
Ultimately, whether non-abstaining, hallucination-free responses can be obtained is controlled by the capacity of the underlying $\prior$: posterior sampling can only shift the mass of the samples towards safer regions if such regions of support exist.

\begin{figure*}[t]
    \centering
    \begin{subfigure}[t]{0.24\textwidth}
        \includegraphics[width=\linewidth]{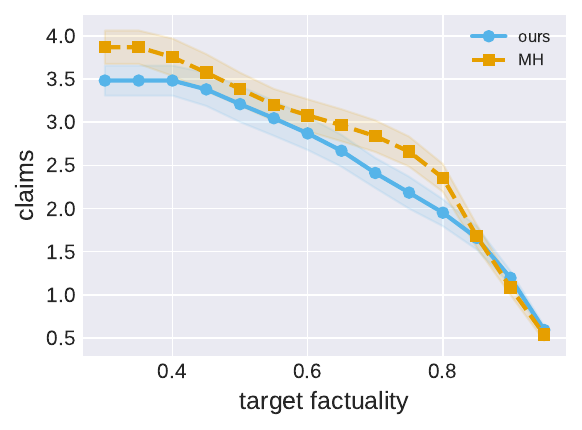}
        \caption{Claims per output.}
    \end{subfigure}\hfill
    \begin{subfigure}[t]{0.24\textwidth}
        \includegraphics[width=\linewidth]{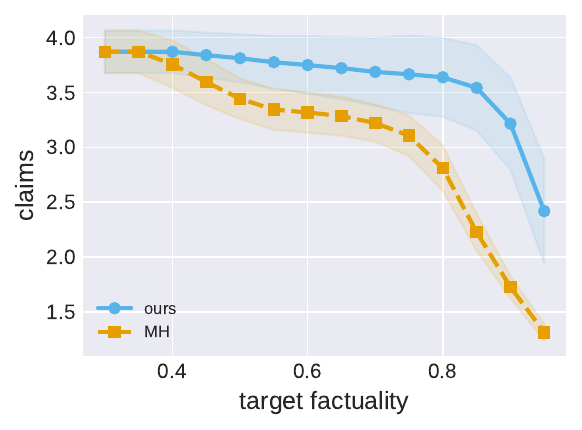}
        \caption{Claims\,$\mid$\,non-abstain.}
    \end{subfigure}\hfill
    \begin{subfigure}[t]{0.24\textwidth}
        \includegraphics[width=\linewidth]{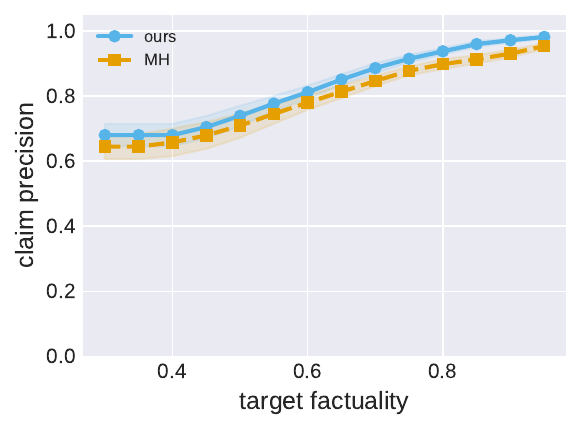}
        \caption{Claim precision.}
    \end{subfigure}\hfill
    \begin{subfigure}[t]{0.24\textwidth}
        \includegraphics[width=\linewidth]{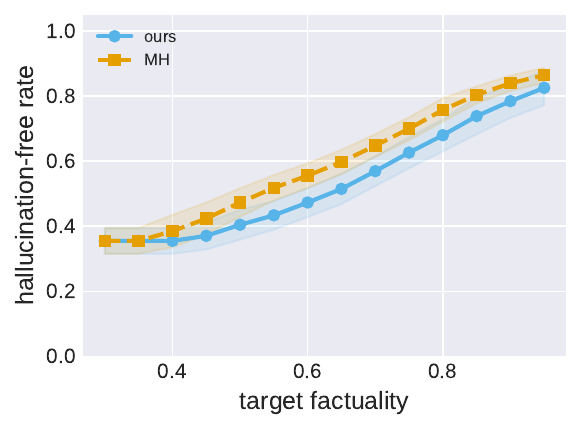}
        \caption{Hallucination-free rate.}
    \end{subfigure}

    \vspace{0.3em}

    \begin{subfigure}[t]{0.24\textwidth}
        \includegraphics[width=\linewidth]{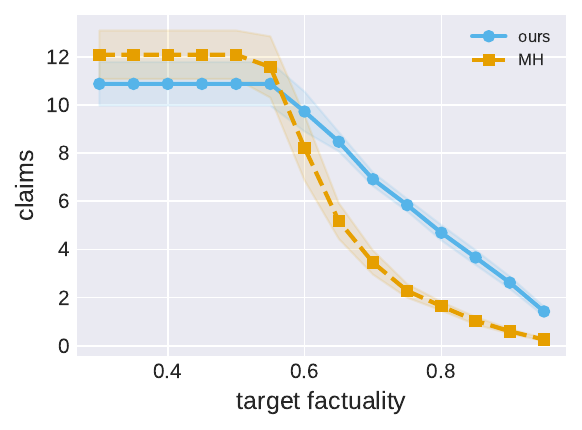}
        \caption{Claims per output.}
    \end{subfigure}\hfill
    \begin{subfigure}[t]{0.24\textwidth}
        \includegraphics[width=\linewidth]{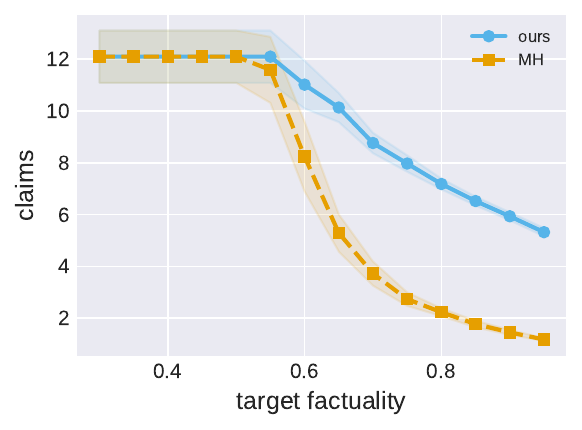}
        \caption{Claims\,$\mid$\,non-abstain.}
    \end{subfigure}\hfill
    \begin{subfigure}[t]{0.24\textwidth}
        \includegraphics[width=\linewidth]{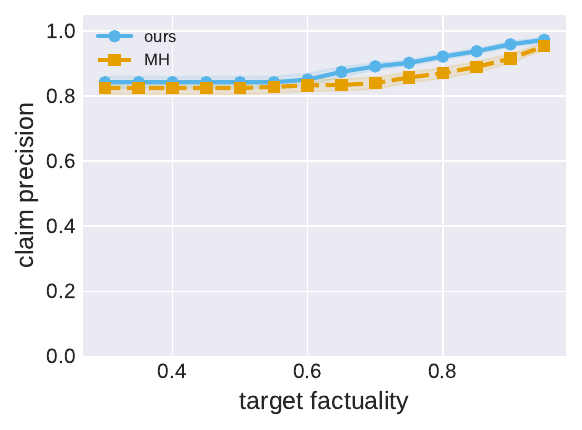}
        \caption{Claim precision.}
    \end{subfigure}\hfill
    \begin{subfigure}[t]{0.24\textwidth}
        \includegraphics[width=\linewidth]{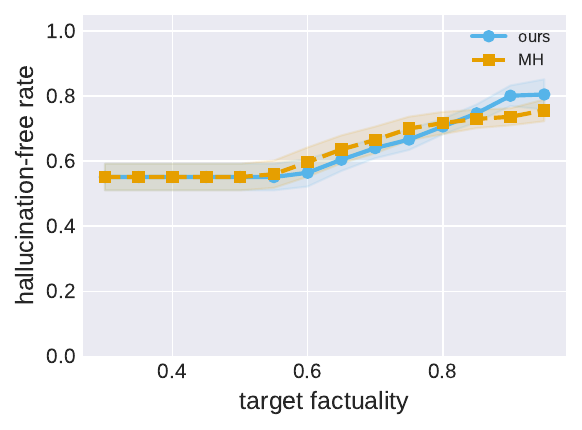}
        \caption{Hallucination-free rate.}
    \end{subfigure}
    \caption{Claim-level proxy metrics for ours (\colorours) vs.\ \texttt{MH} (\colormh) on FActScore (top) and MATH (bottom).}
    \label{fig:utility-metrics}
    \vspace{-10pt}
\end{figure*}

\xhdr{Summary}
While both our method and \texttt{MH} control similar objectives, the resulting downstream utility differs significantly.
As the target level of risk control increases, our method produces more helpful and complete outputs (as judged by an external LM). However, since our sampler cannot stray far from the support of $\lmprior$, instances poorly known to the LLM can only be pushed towards factuality by abstaining, leading to a small uptick in the chance that any given \emph{non-abstaining} answer has at least one incorrect claim in settings where the model lacks the capacity to solve the task.

\vspace{-5pt}
\section{Related Work}
\label{sec:related}

\xhdr{Calibrated stopping rules}
A complementary body of recent work calibrates stopping rules for reasoning models and agents~\citep{Wu2025ThoughtCE, Sadhuka2025EvaluatorRA, Zhou2026OnlineRC}. These methods similarly avoid the limitations of post-hoc filtering (\Cref{section:intro}), but the target of calibration is different: they calibrate decision rules over the reasoning process (when to stop thinking, or when to reject an agent trajectory), whereas we calibrate a family of \emph{output distributions}. On the methodological side, these methods respectively build upon fixed sequence testing~\citep{Angelopoulos2021LearnTT} and E-values~\citep{Ramdas2024HypothesisTW}, and leverage on-policy samples; this is impractical in our setup, where the distribution of scores can be arbitrarily bimodal and no natural grid is available in advance.
Both the goal and the approach presented in this work are thus orthogonal from such efforts.

\xhdr{Particle methods for semantic steering of language models}
Our method is inspired by a line of work that frames semantically constrained generation as posterior sampling~\citep{lew2023sequential,zhao2024probabilistic,loula2025syntactic}.
Similarly to our setup, these works treat a language model as a probabilistic prior and then build a suitably defined posterior distribution.
However, their goal is different than ours, predominantly being concerned with sampling efficiency.
Meanwhile, we focus on how to leverage \emph{unreliable} potentials, a challenge which has not been explored by this literature.

\xhdr{Conformal prediction and risk control under distribution shift} Conformal methods \citep{saunders1999, angelopoulosconformal} have widely been adapted to work under known distribution shift \citep{Tibshirani2019ConformalPU, Angelopoulos2024TheoreticalFO}, tackling the issue via weighted exchangeability.
While we draw inspiration from this approach, we face an additional complication: label distribution shift that is itself directly affected by the calibration parameter. This prevents us from applying methods based on vanilla variants of weighted exchangeability (c.f.\ App.~\ref{appendix:oracles}). We refer to the works of \citet{Fannjiang2022ConformalPU} and \citet{Prinster2024ConformalVG} for a more thorough treatment of conformal guarantees under more general (covariate) shifts.

\section{Discussion: Limitations \& Future Directions}
\label{sec:conclusion}

\xhdr{The $\beta$ parameter} This parameter governs the mixture prior between the LM completions and the explicit abstention response. Its main role is to ensure that the posterior remains well-defined by assigning fixed mass to an always-factual abstention atom, guaranteeing that the normalization constant is bounded away from zero. As a result, $\beta$ controls a \emph{utility-abstention tradeoff}, but not the validity of the procedure on the empirical posterior $\hat{p}^{\hat\tau}$. Larger $\beta$ values provide a stronger fallback path, but shift more probability mass toward abstention and can reduce the informativeness of non-abstaining outputs. Smaller $\beta$ values preserve more mass on LM completions, improving utility when the model is capable, but leave less fallback mass when high-confidence completions are rare. 
Additionally, the choice of $\beta$ and its relative scale compared to $M$ has an effect on the convergence rate of the empirical posterior $\hat{p}^{\hat{\tau}}$ to the population posterior $p^{\hat{\tau}}$ (c.f. Corollary~\ref{corollary:limsup}).

\xhdr{Measuring downstream utility} Our empirical evaluation measures downstream utility using LLM-judge scoring, focusing on properties such as completeness, fluency, and mathematical coherence. While this provides a scalable first validation, it is not a substitute for human evaluation. A natural next step is to complement these experiments with human annotations. Moreover, human preference annotations would shed more light on the trade-off between the utility of responses vs abstentions, as discussed in Section~\ref{sec:experiments}. Such studies would thus help clarify whether the utility improvements measured by LLM judges correspond to user-perceived improvements in realistic deployment settings.

\xhdr{Efficient posterior sampling} Another practical consideration is how posterior particles are sampled. In our experiments, the posterior is approximated using finite sets of base-model completions sampled from the underlying LM, together with the abstention atom. These particles are rescored and reweighted to form an empirical posterior, which is sufficient to validate the finite-sample guarantees proved in the paper. However, this approximation may become inefficient when high-scoring completions are rare or concentrated in low-probability regions of the base distribution. Developing \emph{more efficient} posterior samplers is therefore an important direction for future work.

\bibliographystyle{plainnat} %
\bibliography{clean_references}    %

\appendix

\newpage

\onecolumn
\section{Consequences of Calibrated Scores}
\label{appendix:motivational}
 In this section, we make the informal statement about calibrated forecasts from Section~\ref{sec:motivational} precise. Define the random variable $Z_i := \indicator{\text{claim } i \text{ is conditionally correct} }$, and the indicator $E_i := \indicator{\text{claim } i \text{ ends with an \texttt{EOS} token}}$. We will reason about the filtration $\mathbb{F} = \setof{\calF_i}_{i \in \mathbb{N}}$ where $\calF_i := \sigma(x, \ss_{\leq i}, \mathbf{E}_{\leq i})$. Note that the inclusion of information about end of sequence tokens will allow us to make statements about random stopping times (in particular the random $K$) and can be thought of as a technical artefact void of any deeper meaning.

For shorthand, let $Z_{<i} = \prod_{j=1}^{i-1} Z_j$ and $Z_{\leq i} = Z_{<i} Z_i$.  
Suppose then that the scores were such that
\begin{equation}
	\E[Z_i \sep \calF_i, Z_{<i} = 1] = s_i \quad \quad \text{for } i \geq 1.\label{eq:motivationalassumption:stepwisecalibration} 
\end{equation} 

Intuitively, this means that conditioned on all prior scores and the prior claims being correct, the score $s_i$ is a calibrated predictor of the correctness of the subsequent claim $c_i$. When $k > K$, it is trivially satisfied.

Suppose also that the current score $s_i$ does not give any additional information about the correctness of \emph{past} claims that is not already captured by the preceding scores $\ss_{< i}$; that is,
\begin{equation}
	 \indicator{Z_{< i} = 1} \bot \; s_i, E_i \mid \calF_{i-1}. \label{eq:motivationalassumption:nopeeking}
\end{equation}
Intuitively, this means that the  correctness of a \emph{prefix} can be asserted without looking at the rest of the sentence. We will also write $\indicator{Z_{< i} = 1} \bot \; \calF_{i} \sep \calF_{i-1} $. This essentially assumes that future scores give no additional information about past correctness, if the past scores are available, which we find plausible\footnote{And in any case, this development is purely motivating, and not required for the correctness and calibration guarantees in Section~\ref{sec:calibration}.}. Note also that if $\calF_{k-1}$ reveals that some claim $c_1,\ldots,c_{k-1}$ contains an \texttt{EOS} token, then we assume for mathematical convenience that $s_{k}$ is deterministically $1$ conditional on the information in $\calF_{k-1}$. Then, no information is gained from $\calF_{k}$ either, and the assumption is also trivially satisfied when $k>K$.

This appendix supplies the formal statement and derivation backing the informal \Cref{prop:idealized-informal} in \Cref{sec:motivational}. Concretely:

\begin{proposition}[Formal version of Observation~\ref{prop:idealized-informal}]
\label{prop:idealized-formal}
Let $\{s_i\}_{i \geq 1}$ be the claim-level score sequence produced when sampling $y \sim \lmprior(\cdot \sep x)$, let $K$ denote the random number of claims in $y$, and let $Z_i$ indicate conditional correctness of claim $i$ (\Cref{sec:motivational}). Suppose $s$ satisfies the stepwise-calibration condition~\eqref{eq:motivationalassumption:stepwisecalibration} and the conditional independence condition~\eqref{eq:motivationalassumption:nopeeking}. Then for any $\tau \in [0,1]$,
\[
    \P\!\left(Z_{\leq K} = 1 \;\bigg|\; \textstyle\prod_{i=1}^{K} s_i \geq \tau\right) \geq \tau.
\]
Setting $\tau = 1-\alpha$ gives $\E[\ell(x,y) \mid \Phi^{1-\alpha}(y \mid x) = 1] \leq \alpha$.
\end{proposition}

The remainder of the appendix develops the argument. We pick up the discussion from \Cref{sec:motivational} again. Recall that we defined the $\sigma$-algebras $\calF_i := \sigma(x, \ss_{\leq i}, \mathbf{E}_{\leq i})$. Formally, we define the stopping time $K$ to be
$$
    K = \inf\setof{k \mid E_k = 1}.
$$
Making what we mentioned above about mathematical convenience more precise, we also set $Z_k = 1$, $s_k = 1$ and $c_k = \emptyset$ for $k > K$.  Clearly, $K$ is a valid stopping time, since the event $\setof{K \leq \ell}$ is $\calF_\ell$-measurable, because $\calF_\ell$ includes $\mathbf{E}_{\leq \ell}$ in its definition, and the information contained in the events $E_1,\ldots,E_\ell$ is sufficient to determine whether $K = \ell$ has occurred.

Let us first fix $k$.

\begin{align*}
	\E[Z_{\le k} \mid \calF_k]
	&= \E\!\left[Z_{<k}Z_k \mid \calF_k\right] \\
	&= \E\!\left[Z_{<k}\E[Z_k \mid \calF_k, Z_{<k} ] \mid \calF_k\right] \\
	&\stackrel{(i)}{=} \E\!\left[{Z_{<k}}\E[Z_k \mid \calF_k, Z_{<k}=1] \mid \calF_k\right] \\
	&\stackrel{(ii)}{=}\E\!\left[{Z_{<k} } s_k \mid \calF_k\right] \\
	&= s_k \, \P(Z_{<k}=1 \mid \calF_k) \\
	&\stackrel{(iii)}{=} s_k \, \P(Z_{<k}=1 \mid \calF_{k-1}). 
\end{align*}
We briefly comment on the steps: \begin{enumerate}[(i)]
	\item follows because $$
	\E[Z_k \sep \calF_k, Z_{<k}] = \indicator{A}\E[Z_k \sep \calF_k, A] + \indicator{A^c} \E[Z_k \sep \calF_k, A^c]
	$$
	and for any event $A$ that is determined by $Z_{<k}$ (meaning deterministic conditional on the realization of the random variable $Z_{<k}$). Multiplying by $\indicator{A}$ for $A = \setof{Z_{<k} = 1}$ shows the step.
	\item follows because of \eqref{eq:motivationalassumption:stepwisecalibration}
	\item follows because of \eqref{eq:motivationalassumption:nopeeking}
\end{enumerate}

A proof by induction shows that for any \emph{fixed} $k \geq 1$
\begin{align}
	\E[Z_{\le k} \mid \calF_k] = \textstyle\prod_{i=1}^k s_i.
\end{align}

Since $K$ is a stopping time and $K \leq L$ almost surely\footnote{The maximum length of the LLM response}, we can deduce that for the stopped $\sigma$-algebra,
\begin{align}
	\E[Z_{\le K} \mid \calF_K] = \textstyle\prod_{i=1}^K s_i,
\end{align}
which shows the first part of our claim.

We are now interested in 
$$
	\E\left[\ell(x,y)  \sep \prod_{i=1}^K s_i \geq \tau\right] = \P\left(\ell(x,y) = 1 \sep \prod_{i=1}^K s_i \geq \tau\right) = \P\left(Z_{\leq K} = 0 \sep \prod_{i=1}^K s_i \geq \tau\right).
$$ 
In words, this is the probability of incorrectness of some claim jointly over the claims, when we sample from the posterior distribution conditional on the scores being above the threshold $\tau$. We will now show that this quantity is in turn upper bounded by $1-\tau$, suggesting that calibrated scores lead to the choice $\tau = 1-\alpha$ and guarantee $\alpha$-level risk control.

To see why, let us first rewrite
\begin{align*}
	\P(Z_{\leq K} = 1 \sep \textstyle \prod_{i=1}^K s_i \geq \tau) &= \frac{\P(Z_{\leq K} = 1 \text{ and } \prod_{i=1}^K s_i \geq \tau )}{\P(  \prod_{i=1}^K s_i \geq \tau )}  \\
    &= \frac{\E\left[\indicator{Z_{\leq K} = 1 \text{ and } \prod_{i=1}^K s_i \geq \tau }\right]}{\P(  \prod_{i=1}^K s_i \geq \tau )} \\ &=  \frac{\E\left[ Z_{\leq K} \indicator{\prod_{i=1}^K s_i \geq \tau }\right]}{\P(  \prod_{i=1}^K s_i \geq \tau )} .
\end{align*}
We focus on the numerator.  Since $\prod_{i=1}^K s_i$ is $\calF_K$-measurable for the random stopping time $K$, the event $A_\tau=\setof{\prod_{i=1}^K s_i \geq \tau  }$ is also $\calF_K$-measurable. Therefore
\begin{align*}
	\E[Z_{\leq K} \mathbf{1}_{A_\tau}]
	&= \E\!\left[\E[Z_{\leq K} \mathbf{1}_{A_\tau} \mid \mathcal \calF_K]\right] \\
	&= \E\!\left[\E[Z_{\leq K} \mid \mathcal \calF_K]\mathbf{1}_{A_\tau}\right] \\
	&= \E[\textstyle \prod_{i=1}^K s_i \,\mathbf{1}_{A_\tau}] \\
	&\ge \tau \E[\mathbf{1}_{A_\tau}] \\
	&= \tau \P(A_\tau).
\end{align*}
Dividing by $\P(A_\tau)$ gives
\[
\E[Z_{\leq K} \mid A_\tau] \geq \tau.
\]
Since $Z_{\leq K}$ is binary, and setting $\tau = 1-\alpha$, this is equivalent to 
\[
\P(Z_{\leq K}=1 \mid \Phi^{1-\alpha} (y \mid x) = 1) \ge 1-\alpha.
\]

\section{Deferred Proofs from Section~\ref{sec:calibration}}
\label{appendix:proof}
\subsection{Proof of Proposition~\ref{prop:idealized} (Known $Z^\tau(x)$)}
\label{appendix:proof:idealized}

We collect the auxiliary results needed for the proof, then present the main argument. Along with monotonicity by design, the proof will make use of the following likelihood ratio reweighting:
\begin{fact}
\label{fact:likelihoodratios}
    Fix $x$ and $\tau$. Suppose we draw  $y \sim  p(\cdot \sep x)$ and $y' \sim  p^\tau(\cdot \sep x)$. Then
    \begin{equation}
        \E_{y \sim p(\cdot \sep x)}\left[\frac{\Phi^\tau(y \sep x)}{Z^\tau(x)} \ell(x,y) \right] = \E_{y' \sim p^\tau(\cdot \sep x)}\left[\ell(x,y') \right].
    \end{equation}
\end{fact}

\subsubsection*{Proof of Fact~\ref{fact:likelihoodratios}}

\begin{proof}
    We have
    \begin{align*}
       \E_{y \sim p^\tau(\cdot \sep x)}\left[\ell(x,y)  \right] 
        &= \  \sum_{y \in \calV^*} \ell(x,y)\, p^\tau(y \sep x)    \\
        &=   \sum_{y \in \calV^*} \ell(x,y) \frac{\Phi^\tau(y \sep x)\, p(y \sep x)}{Z^\tau(x)}     \\
        &= \E_{y \sim p(\cdot \sep x)}\left[\ell(x,y) \frac{\Phi^\tau(y \sep x)}{Z^\tau(x)}  \right].
    \end{align*}
\end{proof}

We now turn to the monotonized importance weights $\bar{w}_i(\tau) = \frac{\Phi^\tau(y_i\sep x_i)}{Z^{S(y_i \sep x_i)}(x_i)}$, as well as their relationship to the weights $w_i(\tau) = \frac{\Phi^\tau(y_i\sep x_i)}{Z^{\tau}(x_i)}$. We write $\ell_i = \ell(x_i,y_i)$ as a shorthand in what follows.

\begin{lemma}
\label{lemma:monotonized-weight}
The monotonized weights $\bar{w}_i(\tau) $ satisfy:
\begin{enumerate}[(1)]
    \item $\bar{w}_i(\tau)\ell_i$ is non-increasing in $\tau$.
    \item $\bar{w}_i(\tau)\ell_i \geq w_i(\tau)\ell_i$ for all $\tau$, and hence $\E_{p^0}[\bar{w}_i(\tau)\ell_i] \geq \E_{p^\tau}[\ell_i]$.
    \item $\bar{w}_i(\tau)\ell_i \leq 1/\beta$.
\end{enumerate}
\end{lemma}
\begin{proof}
(1): $\bar{w}_i(\tau)\ell_i = \frac{\ell_i}{Z^{S(y_i \sep x_i)}(x_i)} \mathbf{1}\{S(y_i \sep x_i) \geq \tau\}$ is constant for $\tau \leq S(y_i \sep x_i)$ and zero for $\tau > S(y_i \sep x_i)$. (2): For $\tau \leq S(y_i \sep x_i)$, $Z^{S(y_i \sep x_i)}(x_i) \leq Z^\tau(x_i)$ (since $Z^\tau$ is non-increasing), so $\bar{w}_i(\tau) \geq w_i(\tau)$; for $\tau > S(y_i \sep x_i)$, both are zero. The expectation bound follows from Fact~\ref{fact:likelihoodratios}. (3): $Z^{S(y_i \sep x_i)}(x_i) \geq \beta$ since the abstention atom always survives, and $\ell, \Phi^\tau \in \setof{0,1}$.
\end{proof}

The proof of Proposition~\ref{prop:idealized} assumes a hallucinated completion from the mixture prior on the test prompt. Formally, we assume that $y_{n+1} \stackrel{\mathrm{i.i.d.}}{\sim} p(\cdot \sep \xtest)$ conditionally independently of the sample $\ytest$ given $\xtest$. We also denote $x_{n+1} = \xtest$ for convenience. The selection rule $\idealizedrule$ is then compared to the ``oracle'' one that assumes access to a hallucinated base completion:
\newcommand{\idealizedoraclerule}{\tau_{\mathrm{id}}^*}
\begin{align}
    \idealizedoraclerule = \inf\setof{\tau \sep \frac{1}{n+1} \sum_{i=1}^{n+1} \frac{\Phi^\tau(y_i \sep x_i)}{Z^{S(y_i \sep x_i)}(x_i)}\ell(x_i,y_i)  \leq \alpha }. \label{eq:noiselessoracle}
\end{align}

\begin{lemma}
\label{lemma:tau-ordering-idealized}
    $\idealizedrule \geq \idealizedoraclerule$ almost surely.
\end{lemma}
\begin{proof}
    For any $\tau$ in the feasible set of $\idealizedrule$, we have
    $\frac{1}{n+1}\sum_{i=1}^{n} \bar{w}_i(\tau)\ell_i + \frac{1}{\beta(n+1)} \leq \alpha$.
    Since $\bar{w}_{n+1}(\tau)\ell_{n+1} \leq 1/\beta$ by Lemma~\ref{lemma:monotonized-weight}:
    $$
        \frac{1}{n+1}\sum_{i=1}^{n+1} \bar{w}_i(\tau)\ell_i \leq \frac{1}{n+1}\sum_{i=1}^{n} \bar{w}_i(\tau)\ell_i + \frac{1}{\beta(n+1)} \leq \alpha.
    $$
    So $\tau$ is in the feasible set of $\idealizedoraclerule$, giving $\idealizedoraclerule \leq \idealizedrule$ a.s.
\end{proof}

\subsubsection*{Main Proof of Proposition~\ref{prop:idealized}}

\begin{proof}
\medskip\noindent\textit{Step 1: Reduce to monotonized weights at $\idealizedrule$.}
Recall that the test response $\ytest$ is drawn from $p^{\idealizedrule}(\cdot \sep \xtest)$ for a fresh test prompt $\xtest \sim \covdist$ exchangeable with the calibration prompts, where $\idealizedrule = \idealizedrule(D_{1:n})$ depends only on the calibration set. Let us denote by $x_{n+1} = \xtest$ and let $y_{n+1} \sim p(\cdot \sep x_{n+1})$ be a \emph{hallucinated} sample from the mixture prior. Since $\idealizedrule$ is chosen without knowledge of the realization of $(x_{n+1}, y_{n+1})$, Fact~\ref{fact:likelihoodratios} gives $\E[\ell(\xtest,\ytest) \sep \idealizedrule, \xtest] =  \E[w_{n+1}(\idealizedrule)\ell(x_{n+1}, y_{n+1}) \sep \idealizedrule, \xtest]$. By Lemma~\ref{lemma:monotonized-weight}, $\bar{w}_{n+1}(\tau) \geq w_{n+1}(\tau)$ pointwise, so
\begin{equation*}
    \E[\ell(\xtest, \ytest) \sep \idealizedrule] \leq \E[\bar{w}_{n+1}(\idealizedrule)\,\ell_{n+1} \sep \idealizedrule],  \quad \text{a.s.}
\end{equation*}
From this
\begin{equation}
    \E[\ell(\xtest, \ytest)] \leq \E[\bar{w}_{n+1}(\idealizedrule)\,\ell_{n+1}]\label{eq:reduced-to-monotonized}
\end{equation}
where we took the expectation over the randomness in $\idealizedrule$ again.

\medskip\noindent\textit{Step 2: Transfer from $\idealizedrule$ to $\idealizedoraclerule$.}
By Lemma~\ref{lemma:tau-ordering-idealized}, $\idealizedrule \geq \idealizedoraclerule$ a.s. Since by Lemma~\ref{lemma:monotonized-weight} $\bar{w}_{n+1}(\tau)\ell_{n+1}$ is non-increasing in $\tau$:
\begin{equation}
    \bar{w}_{n+1}(\idealizedrule)\,\ell_{n+1} \leq \bar{w}_{n+1}(\idealizedoraclerule)\,\ell_{n+1}, \quad \text{a.s.} \label{eq:almostsuremonotone:idealized}
\end{equation}
\medskip\noindent\textit{Step 3: Exchangeability.}
The hallucinated calibration set $D_{1:n+1} = \{(x_i, y_i)\}_{i=1}^{n+1}$ is exchangeable\footnote{Indeed, we assumed the prompts were exchangeable, and the responses were  conditionally independent draws from $p(\cdot \sep x_i)$.}, and $\idealizedoraclerule = \idealizedoraclerule(D_{1:n+1})$ is permutation invariant. Therefore, conditional on the multiset $E = \setof{(x_i,y_i)}_{i \in [n+1]}$, we have 
$$
    \E[\bar{w}_{n+1}(\idealizedoraclerule)\,\ell_{n+1} \sep E] = \frac{1}{n+1}\sum_{i=1}^{n+1} \bar{w}_i(\idealizedoraclerule) \ell_i \leq \alpha,
$$
where inequality uses the definition of $\idealizedoraclerule$ in \eqref{eq:noiselessoracle}. Integrating out $E$ again gives us
$$
    \E[\bar{w}_{n+1}(\idealizedoraclerule)\,\ell_{n+1}] = \E\!\left[\frac{1}{n+1}\sum_{i=1}^{n+1} \bar{w}_i(\idealizedoraclerule)\,\ell_i\right] \leq \alpha.
$$
Finally, we combine the pieces to obtain $$\E[\ell(\xtest, \ytest)] \stackrel{\eqref{eq:reduced-to-monotonized}}{\leq} \E[\bar{w}_{n+1}(\idealizedrule)\,\ell_{n+1}] \stackrel{\eqref{eq:almostsuremonotone:idealized}}{\leq} \E[\bar{w}_{n+1}(\idealizedoraclerule)\,\ell_{n+1}] \leq \alpha.$$
\end{proof}

\subsection{Proof of Theorem~\ref{thm:main} (Empirical Posterior Calibration)}
\label{appendix:proof:theorem32}

We first give an equivalent of Fact~\ref{fact:likelihoodratios} for this setting.
\begin{lemma}[Risk under Empirical Posterior]
\label{lemma:factanalogue}
Let $\gamma$ be any $[0,1]$-valued random variable. For any $i$, suppose that conditional on observing $(x_i, T_i, \gamma)$ we draw $y \sim \hat p^\gamma(\cdot \sep x_i, T_i)$. Then
$$
    \E[\ell(x_i, y) \sep x_i, T_i, \gamma] = \hat H_i(\gamma).
$$
\end{lemma}
\begin{proof}
Conditional on $(x_i, T_i, \gamma)$, the distribution $\hat p^\gamma(\cdot \sep x_i, T_i)$ is finitely supported on $\{y_{i1}, \ldots, y_{iM}, \texttt{abstain}\}$ with mass $\frac{(1-\beta)\Phi^\gamma(y_{ij}\sep x_i)/M}{\hat Z^\gamma(x_i, T_i)}$ on each $y_{ij}$ and mass $\frac{\beta}{\hat Z^\gamma(x_i, T_i)}$ on $\texttt{abstain}$. Summing $\ell$ against these weights and using $\ell(x_i, \texttt{abstain}) = 0$ yields the result.
\end{proof}

\begin{lemma}
\label{lemma:monotone-h}
The monotonized estimator $\bar{H}_i(\tau) = \sup_{\tau' \geq \tau} \hat{H}_i(\tau')$ satisfies:
\begin{enumerate}[(1)]
    \item $\bar{H}_i(\tau)$ is non-increasing in $\tau$.
    \item $\bar{H}_i(\tau) \geq \hat{H}_i(\tau)$ for all $\tau$.
    \item $\bar{H}_i(\tau) \leq 1$.
\end{enumerate}
\end{lemma}
\begin{proof}
(1) If $\tau_1 < \tau_2$, then $\bar H_i(\tau_1) = \sup_{\tau' \geq \tau_1} \hat H_i(\tau') \geq \sup_{\tau' \geq \tau_2} \hat H_i(\tau') = \bar H_i(\tau_2)$. (2) Immediate. (3) The numerator of \eqref{eq:empirical-conditional-risk} is bounded above by $(1-\beta)\frac{1}{M}\sum_j \Phi^\tau(y_{ij}\sep x_i) \leq \hat Z^\tau(x_i, T_i)$, so $\hat H_i(\tau) \leq 1$, and a supremum of values in $[0,1]$ is in $[0,1]$.
\end{proof}

\begin{lemma}
\label{lemma:tau-ordering-estimated}
Let us refer to the test prompt $\xtest$ and tuple $T_{\mathrm{test}}$ as $x_{n+1}$ and $T_{n+1}$ respectively. Define the oracle
\begin{equation}
    \tau^\star = \inf\setof{\tau \sep \frac{1}{n+1}\sum_{i=1}^{n+1} \bar{H}_i(\tau) \leq \alpha}. \label{eq:oracle-estimated}
\end{equation}
Then $\hat\tau \geq \tau^\star$ almost surely.
\end{lemma}
\begin{proof}
For any $\tau$ in the feasible set of $\hat\tau$, \Cref{algo:crc} guarantees
$\frac{1}{n+1}\sum_{i=1}^{n} \bar H_i(\tau) + \frac{1}{n+1} \leq \alpha$.
Since by Lemma~\ref{lemma:monotone-h}, $\bar H_{n+1}(\tau) \leq 1$,
$$
    \frac{1}{n+1}\sum_{i=1}^{n+1} \bar H_i(\tau) \leq \frac{1}{n+1}\sum_{i=1}^{n} \bar H_i(\tau) + \frac{1}{n+1} \leq \alpha,
$$
so $\tau$ is in the feasible set of $\tau^\star$, giving $\tau^\star \leq \hat\tau$ a.s.
\end{proof}

\subsubsection*{Main Proof of Theorem~\ref{thm:main}}
The proof follows the general approach for conformal risk control proofs \cite{angelopoulosconformal}
\begin{proof}
The test response is $\ytest \sim \hat p^{\hat\tau}(\cdot \sep \xtest, T_{\mathrm{test}})$. Write $(x_{n+1}, T_{n+1}) := (\xtest, T_{\mathrm{test}})$ for the test tuple. Since the prompts $(x_i)_{i=1}^{n+1}$ are exchangeable and, conditional on these prompts, the particles in each $T_i$ are drawn i.i.d.\ from $\lmprior(\cdot \sep x_i)$, the tuples $(x_i, T_i)_{i=1}^{n+1}$ are exchangeable as well.

\medskip\noindent\textit{Step 1: Empirical Posterior Risk at $\hat\tau$.}
When $\ytest \sim \hat p^{\hat\tau}(\cdot \sep \xtest, T_{\mathrm{test}})$, Lemma~\ref{lemma:factanalogue} applied with $\gamma = \hat\tau$ gives
$$
    \E[\ell(\xtest, \ytest) \sep \xtest, T_{\mathrm{test}}, \hat\tau] = \hat H_{n+1}(\hat\tau).
$$
Combining with Lemma~\ref{lemma:monotone-h} (2) and taking the outer expectation,
\begin{equation}
    \E[\ell(\xtest, \ytest)] = \E[\hat H_{n+1}(\hat\tau)] \leq \E[\bar H_{n+1}(\hat\tau)]. \label{eq:pessimism-clean}
\end{equation}

\medskip\noindent\textit{Step 2: Transfer from $\hat\tau$ to $\tau^\star$.}
By Lemma~\ref{lemma:tau-ordering-estimated}, $\hat\tau \geq \tau^\star$ a.s. Since $\bar H_{n+1}(\tau)$ is non-increasing in $\tau$, by Lemma~\ref{lemma:monotone-h},
\begin{equation}
    \bar H_{n+1}(\hat\tau) \leq \bar H_{n+1}(\tau^\star), \quad \text{a.s.} \label{eq:monotone-transfer}
\end{equation}

\medskip\noindent\textit{Step 3: Exchangeability.}
The tuples $(x_i, T_i)_{i=1}^{n+1}$ are exchangeable, and $\tau^\star = \tau^\star((x_i, T_i)_{i=1}^{n+1})$ is a symmetric function of these tuples. Conditioning on the multiset $E = \{(x_i, T_i)\}_{i \in [n+1]}$, exchangeability gives
$$
    \E[\bar H_{n+1}(\tau^\star) \sep E] = \frac{1}{n+1}\sum_{i=1}^{n+1} \bar H_i(\tau^\star) \leq \alpha, \quad \text{a.s.},
$$
by the definition of $\tau^\star$ in \eqref{eq:oracle-estimated}. Integrating over $E$,
\begin{equation}
    \E[\bar H_{n+1}(\tau^\star)] \leq \alpha. \label{eq:exchangeability-bound}
\end{equation}

Combining \eqref{eq:pessimism-clean}, \eqref{eq:monotone-transfer}, and \eqref{eq:exchangeability-bound} yields $\E[\ell(\xtest, \ytest)] \leq \alpha$. \qedhere
\end{proof}

\subsection{Proof of Lemma~\ref{lemma:glivenko}}
\label{appendix:empirical-to-population}

Before proving the lemma, we will need a technical lemma bounding the bracketing number
\cite{wellner} of some function classes. Let us denote briefly by
\(P_{Y,S \mid x}\) the conditional joint distribution of responses and scores given
a prompt \(x\).\footnote{We only model the joint distribution to handle randomized
scorers; for deterministic scorers this is not necessary and a measure over
\(\mathcal V^*\) suffices.}

\begin{lemma}
    \label{lemma:bracketing}
    Assume the distribution of the scores is continuous\footnote{We enforce this by adding noise to the elicited scores.}. Let \(h: \mathcal{V}^* \times \mathcal{S} \rightarrow \{0,1\}\) be a fixed
    measurable map. Define the function class
    \[
        \mathcal{G}
        =
        \left\{
        (y,s) \mapsto h(y,s)\indicator{s \geq \tau}
        \;\middle|\;
        \tau \in [0,1]
        \right\},
    \]
    and identify its elements as \(g_\tau\). Then, for every \(\epsilon>0\),
    \[
        N_{[\,]}\!\left(
        \epsilon,
        \mathcal G,
        \norm{\cdot}_{L_1(P_{Y,S\mid x})}
        \right)
        < \infty .
    \]
\end{lemma}

\begin{proof}
Let \(P=P_{Y,S\mid x}\). For \(\tau\in[0,1]\), write
\[
    g_\tau(y,s)
    =
    h(y,s)\indicator{s\ge \tau}.
\]
Fix \(\tau_1\le \tau_2\). Then for every \(\tau\in[\tau_1,\tau_2]\),
\[
    g_{\tau_2}(y,s)
    =
    h(y,s)\indicator{s\ge \tau_2}
    \le
    h(y,s)\indicator{s\ge \tau}
    \le
    h(y,s)\indicator{s\ge \tau_1}
    =
    g_{\tau_1}(y,s),
\]
where we used that \(h\ge 0\). Hence the bracket
\[
    [g_{\tau_2},g_{\tau_1}]
\]
contains all functions \(g_\tau\) with \(\tau\in[\tau_1,\tau_2]\).

The \(L_1(P)\)-width of this bracket is
\begin{align*}
    \norm{g_{\tau_1}-g_{\tau_2}}_{L_1(P)}
    &=
    \mathbb E_P\left[
        h(Y,S)
        \left(
        \indicator{S\ge \tau_1}
        -
        \indicator{S\ge \tau_2}
        \right)
    \right] \\
    &=
    \mathbb E_P\left[
        h(Y,S)\indicator{\tau_1\le S<\tau_2}
    \right].
\end{align*}
Define the finite measure \(\mu_h\) on \([0,1]\) by
\[
    \mu_h(A)
    =
    \mathbb E_P\left[
        h(Y,S)\indicator{S\in A}
    \right].
\]
Since \(h\in\{0,1\}\), we have \(\mu_h([0,1])\le 1\). With this notation,
\[
    \norm{g_{\tau_1}-g_{\tau_2}}_{L_1(P)}
    =
    \mu_h([\tau_1,\tau_2)).
\]

Now fix \(\epsilon>0\). Since \(\mu_h\) is a finite measure on \([0,1]\), we can
partition \([0,1]\) into finitely many intervals \(I_k=[a_k,a_{k+1})\) such that each interval has \(\mu_h\)-mass at
most \(\epsilon\). Since the scores are continuous, this measure has no atoms. For each such interval \(I_k=[a_k,a_{k+1})\), the bracket
\[
    [g_{a_{k+1}},g_{a_k}]
\]
has \(L_1(P)\)-width
\[
    \norm{g_{a_k}-g_{a_{k+1}}}_{L_1(P)}
    =
    \mu_h([a_k,a_{k+1}))
    \le \epsilon.
\]
Thus finitely many \(L_1(P)\)-brackets of width at most \(\epsilon\) cover
\(\mathcal G\). Therefore
\[
    N_{[\,]}\!\left(
    \epsilon,
    \mathcal G,
    \norm{\cdot}_{L_1(P_{Y,S\mid x})}
    \right)
    <\infty.
\]
\end{proof}

We are ready to prove Lemma~\ref{lemma:glivenko}.

Fix a prompt $x$ and abstention mass $\beta > 0$. Write $T_M = (y_1, \ldots, y_M)$ with $y_j \stackrel{\text{i.i.d.}}{\sim} \lmprior(\cdot \sep x)$. Recall $\Phi^\tau(y \sep x) = \indicator{S(y \sep x) \geq \tau}$ where $S(y \sep x) = \prod_i s_i$ is the completion-level score. Define the quantities
$$
    \hat A_M(\tau) := \frac{1}{M}\sum_{j=1}^{M} \Phi^\tau(y_j \sep x)\, \ell(x, y_j), \qquad A(\tau) := \E_{y \sim p_{\llm}}[\Phi^\tau(y \sep x)\, \ell(x, y)],
$$
$$
    \hat B_M(\tau) := \frac{1}{M}\sum_{j=1}^{M} \Phi^\tau(y_j \sep x), \qquad B(\tau) := \E_{y \sim p_{\llm}}[\Phi^\tau(y \sep x)].
$$
 The empirical and population mixture normalizers are
$$
    \hat Z^\tau_M = \beta + (1-\beta)\hat B_M(\tau), \qquad Z^\tau = \beta + (1-\beta) B(\tau),
$$
both in $[\beta, 1]$. We dropped the dependence on $x$, as we fixed this variable already.

\medskip\noindent\textit{Step 1: Uniform convergence via Glivenko--Cantelli.}
Instantiating $h(y,s) = \ell(x,y)$ and $h(y,s) = 1$ shows that $\{y \mapsto 1\{S(y\mid x)\ge \tau\}:\tau\in[0,1]\}$ and $
\{y\mapsto \ell(x,y)1\{S(y\mid x)\ge\tau\}:\tau\in[0,1]\}
$ have finite Bracketing number, by Lemma~\ref{lemma:bracketing}. By Theorem 2.4.1. in \cite{wellner}, both of these classes are therefore Glivenko-Cantelli. Therefore
\begin{equation}
    \sup_{\tau \in [0,1]} \abs{\hat A_M(\tau) - A(\tau)} \xrightarrow[M\to\infty]{\text{a.s.}} 0,
    \qquad
    \sup_{\tau \in [0,1]} \abs{\hat B_M(\tau) - B(\tau)} \xrightarrow[M\to\infty]{\text{a.s.}} 0. \label{eq:gc-uniform}
\end{equation}

\medskip\noindent\textit{Step 2: Uniform convergence of the empirical-posterior risk.}
Recall that the empirical and population conditional risks are given by
$$
\hat R_M = \E_{y \sim \hat p^\tau(\cdot \sep x, T_M)}[\ell(x, y)] \qquad \text{and} \qquad R_M = \E_{y \sim p^\tau(\cdot \sep x)}[\ell(x, y)].
$$
This can be rewritten as
$$
    \hat R_M(x, \tau) = \frac{(1-\beta)\hat A_M(\tau)}{\hat Z^\tau_M},
    \qquad
    R(x, \tau) = \frac{(1-\beta) A(\tau)}{Z^\tau}.
$$

Adding and subtracting \((1-\beta)\hat A_M(\tau)/Z^\tau\), we get
\begin{align*}
    \abs{\hat R_M(x,\tau) - R(x,\tau)}
    &=
    \abs{
    \frac{(1-\beta)\hat A_M(\tau)}{\hat Z^\tau_M}
    -
    \frac{(1-\beta)A(\tau)}{Z^\tau}
    } \\
    &=
    \abs{
    \frac{(1-\beta)\hat A_M(\tau)}{\hat Z^\tau_M}
    -
    \frac{(1-\beta)\hat A_M(\tau)}{Z^\tau}
    +
    \frac{(1-\beta)\hat A_M(\tau)}{Z^\tau}
    -
    \frac{(1-\beta)A(\tau)}{Z^\tau}
    } \\
    &\leq
    \abs{
    \frac{(1-\beta)\hat A_M(\tau)}{Z^\tau}
    -
    \frac{(1-\beta)A(\tau)}{Z^\tau}
    }
    +
    \abs{
    \frac{(1-\beta)\hat A_M(\tau)}{\hat Z^\tau_M}
    -
    \frac{(1-\beta)\hat A_M(\tau)}{Z^\tau}
    } \\
    &=
    \frac{1-\beta}{Z^\tau}\abs{\hat A_M(\tau)-A(\tau)}
    +
    (1-\beta)\hat A_M(\tau)
    \abs{\frac{1}{\hat Z^\tau_M}-\frac{1}{Z^\tau}} \\
    &=
    \frac{1-\beta}{Z^\tau}\abs{\hat A_M(\tau)-A(\tau)}
    +
    (1-\beta)\hat A_M(\tau)
    \frac{\abs{\hat Z^\tau_M-Z^\tau}}{\hat Z^\tau_M Z^\tau}.
\end{align*}
Since
\[
    \hat Z^\tau_M-Z^\tau
    =
    \bigl[\beta+(1-\beta)\hat B_M(\tau)\bigr]
    -
    \bigl[\beta+(1-\beta)B(\tau)\bigr]
    =
    (1-\beta)(\hat B_M(\tau)-B(\tau)),
\]
and since \(\hat A_M(\tau)\le 1\), \(Z^\tau\ge \beta\), and
\(\hat Z^\tau_M Z^\tau\ge \beta^2\), it follows that
\begin{align*}
    \abs{\hat R_M(x,\tau)-R(x,\tau)}
    &\leq
    \frac{1-\beta}{\beta}\abs{\hat A_M(\tau)-A(\tau)}
    +
    \frac{(1-\beta)^2}{\beta^2}
    \abs{\hat B_M(\tau)-B(\tau)}.
\end{align*}

Taking the supremum over $\tau \in [0,1]$ on both sides and applying \eqref{eq:gc-uniform},
\begin{equation}
    \sup_{\tau \in [0,1]} \abs{\hat R_M(x, \tau) - R(x, \tau)} \xrightarrow[M\to\infty]{\text{a.s.}} 0. \label{eq:uniform-risk}
\end{equation}
The same argument applies for $\covdist$-a.e.\ $x$, completing the proof of Lemma~\ref{lemma:glivenko}.

\subsection{Proof of Corollary~\ref{corollary:limsup}}
Since the guarantee above holds uniformly over $\tau$, it in particular holds for the selected threshold $\hat\tau_M \in [0,1]$ from Algorithm~\ref{algo:crc} when run with $M$ samples from $p_{\llm}$:
$$
    \abs{\hat R_M(\xtest, \hat\tau_M) - R(\xtest, \hat\tau_M)} \leq \sup_{\tau \in [0,1]} \abs{\hat R_M(\xtest, \tau) - R(\xtest, \tau)}.
$$
So by \eqref{eq:uniform-risk} and by the bounded convergence theorem,
$$
    \E\!\left[\sup_{\tau \in [0,1]} \abs{\hat R_M(\xtest, \tau) - R(\xtest, \tau)}\right] \xrightarrow[M\to\infty]{} 0.
$$

By Theorem 3.2,
\[
\mathbb E[\hat R_M(x_{\rm test},\hat\tau_M)]\le \alpha.
\]
Therefore, for \(y_{\rm test}^{\rm pop}\sim p^{\hat\tau_M}(\cdot\mid x_{\rm test})\),
\[
\mathbb E[\ell(x_{\rm test},y_{\rm test}^{\rm pop})]
=
\mathbb E[R(x_{\rm test},\hat\tau_M)]
\le
\alpha+o(1),
\]
and hence
\[
\limsup_{M\to\infty}
\mathbb E[\ell(x_{\rm test},y_{\rm test}^{\rm pop})]\le \alpha.
\]

\section{Design Choices for the Conformal Routine}
\label{appendix:oracles}

This appendix provides motivation and design rationale for the calibration rule in Section~\ref{sec:calibration}. The arguments here are not needed for the validity proofs of Proposition~\ref{prop:idealized} and Theorem~\ref{thm:main}. Instead, they explain how the rule arises from an idealized known-normalizer setting and why we use a fixed-denominator, monotonized objective. We highlight the connections to traditional approaches to conformal prediction with (fixed) covariate and label shift. %

The calibration procedures developed in Section~\ref{sec:calibration} are intuitively derived from the oracles $\idealizedoraclerule$ and $\tau^\star$ in \eqref{eq:noiselessoracle} and \eqref{eq:oracle-estimated} respectively, both of which use a \emph{fixed} denominator of $n+1$ and importance-reweights prior samples to evaluate/estimate the posterior risk. A natural alternative is to use an \emph{adaptive} (self-normalized) denominator, as is common in the distribution-shift conformal prediction literature \citep{Tibshirani2019ConformalPU, Angelopoulos2024TheoreticalFO, Fannjiang2022ConformalPU, Prinster2024ConformalVG}. Fundamentally, this oracle arises as follows. Suppose that $U_i \sim P$ for $i \in [n]$ and $U_{n+1} \sim Q$. Then conditional on the multiset of observed values $E = \setof{U_1,\ldots,U_n,U_{n+1}}$ 
$$
    U_{n+1} \sep E \sim \sum_{i=1}^{n+1} v_i \delta_{U_i} \quad \text{where} \quad v_i = \frac{\frac{dQ(U_i)}{dP(U_i)}}{\sum_{i=1}^{n+1} \frac{dQ(U_i)}{dP(U_i)}}.
$$
This can be found in Proposition 7.6. of \cite{Angelopoulos2024TheoreticalFO}. Extensions exist under Feedback covariate shift \cite{Fannjiang2022ConformalPU} and more arbitrary dependencies \cite{Prinster2024ConformalVG}. As a consequence, if we had $x_i \sim \covdist$ for $i \in [n+1]$ and $y_{n+1} \sim p^\tau(\cdot \sep x_{n+1})$ for \emph{fixed} $\tau$, we'd have
$$
    (x_{n+1},y_{n+1}) \sep E \sim \sum_{i=1}^{n+1} v_i \delta_{(x_i,y_i)} \quad \text{ where } v_i = \frac{\frac{\Phi^\tau(y_i \sep x_i)}{Z^\tau(x_i)}}{\sum_{j=1}^{n+1} \frac{\Phi^\tau(y_j \sep x_j)}{Z^\tau(x_j)}}
$$

In the known $Z^\tau$ case, this could suggest\footnote{Ignoring for a second that monotonizing this objective is also trickier than if the denominator is fixed.} an oracle selection rule 
$$
\tau^*_{\mathrm{SN}} = \inf\setof{\tau \sep \frac{\sum_{i=1}^{n+1} \frac{\Phi^\tau(y_i \sep x_i)}{Z^\tau(x_i)}\ell(x_i,y_i) }{\sum_{i=1}^{n+1} \frac{\Phi^\tau(y_i \sep x_i)}{Z^\tau(x_i)}} \leq \alpha},
$$
Intuitively, $\tau^*_{\mathrm{SN}}$ computes a weighted average of the losses, with weights proportional to each point's likelihood under the posterior $p^\tau$. One issue however, is that this is ill defined. For the weighted exchangeability to make sense, $y_{n+1}$ now needs to be sampled according to $p^\tau$. However, $\tau$ depends on $y_{n+1}$. This creates a circular dependency, that the choice of $\idealizedoraclerule$ and $\tau^\star$ is able to sidestep, by depending only on a hallucinated calibration set completions from the \emph{prior}.

\subsubsection*{Population-Level Equivalence}
We might still be curious about why $\idealizedoraclerule$ and $\tau_{\mathrm{SN}}^*$ are different, even though both seem principled. In fact, the expected denominator in $\tau^*_{\mathrm{SN}}$ coincides with that in $\idealizedoraclerule$. Writing $w_i(\tau) = \Phi^\tau(y_i \mid x_i) / Z^\tau(x_i)$, a direct calculation using the tower rule gives
\begin{align*}
    \E\!\left[\sum_{i=1}^{n+1} w_i(\tau)\right]
    &= \sum_{i=1}^{n+1} \E\!\left[\E\!\left[\frac{\Phi^\tau(Y_i \mid X_i)}{Z^\tau(X_i)} \,\bigg|\, X_i\right]\right]
    = \sum_{i=1}^{n+1} \E\!\left[\frac{Z^\tau(X_i)}{Z^\tau(X_i)}\right]
    = n+1.
\end{align*}
Thus, in expectation, the self-normalized denominator equals $n+1$. The two oracles can therefore be expected to behave similarly in large samples. Indeed, empirical experimentation has shown us that their resulting approximations behave very similarly in practice too. As our theoretical analysis employs $\idealizedoraclerule$, our Algorithm~\ref{algo:crc} also uses a fixed denominator.

\section{Extensive Experimental Results}
\label{appendix:add-results}

\subsection{MATH Calibration Validity}
\label{appendix:math-calibration}

\Cref{fig:calibration-appendix} shows the results of the empirical validity experiment (\Cref{fig:calibration}) for both FActScore and MATH. MATH mirrors FActScore: empirically, valid coverage is achieved at all levels, with a flatlining region at targets low enough that both methods degrade to sampling from the prior.

\begin{figure}[h]
    \centering
    \begin{subfigure}[b]{0.45\textwidth}
        \includegraphics[width=\linewidth]{full-paper/figures/factscore_pdfs/factual_or_abstain.pdf}
        \caption{FActScore.}
        \label{fig:calibration-factscore-appendix}
    \end{subfigure}\hfill
    \begin{subfigure}[b]{0.45\textwidth}
        \includegraphics[width=\linewidth]{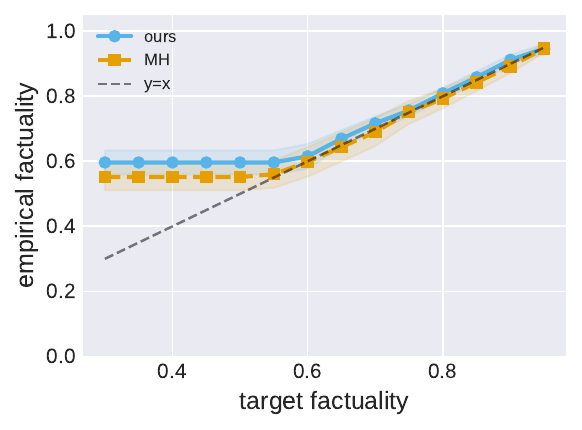}
        \caption{MATH.}
        \label{fig:calibration-math}
    \end{subfigure}
    \caption{Target vs.\ observed factuality on FActScore (a) and MATH (b); ours (blue) and \texttt{MH} (orange). Both methods produce empirically valid coverage on both tasks. Averaged over 10 splits; shaded area shows 95\% CIs.}
    \label{fig:calibration-appendix}
\end{figure}

\subsection{Per-statement Likert Scores}
\label{appendix:likert-axes}

\Cref{fig:utility-likert-axes} breaks down the FActScore downstream utility results from \Cref{fig:utility-downstream} across the three axes (see Appendix~\ref{appendix:judge-prompts}). These results reveal that even at high targets, the LM judge considers non-abstaining outputs from both methods to be fluent and coherent. However, as the target factuality increases, outputs from \texttt{MH} quickly degrade in terms of their completeness and helpfulness. Note that \Cref{fig:utility-downstream} is the arithmetic mean of these axes.

\begin{figure}[h!]
    \centering
    \begin{subfigure}[b]{0.31\textwidth}
        \includegraphics[width=\linewidth]{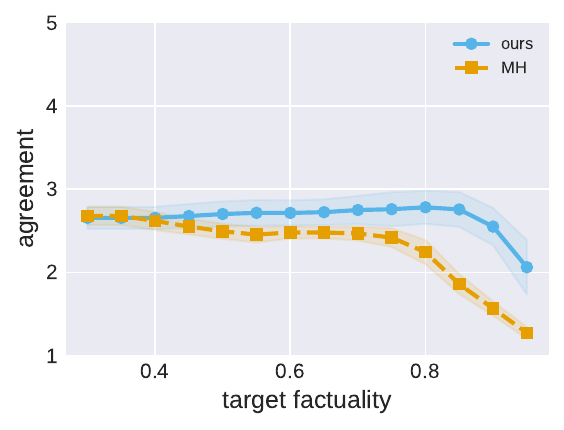}
        \caption{Completeness\,$\mid$\,non-abstaining.}
    \end{subfigure}\hfill
    \begin{subfigure}[b]{0.31\textwidth}
        \includegraphics[width=\linewidth]{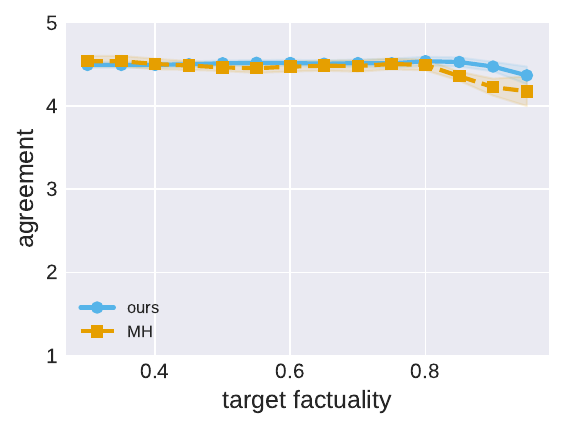}
        \caption{Fluency\,$\mid$\,non-abstaining.}
    \end{subfigure}\hfill
    \begin{subfigure}[b]{0.31\textwidth}
        \includegraphics[width=\linewidth]{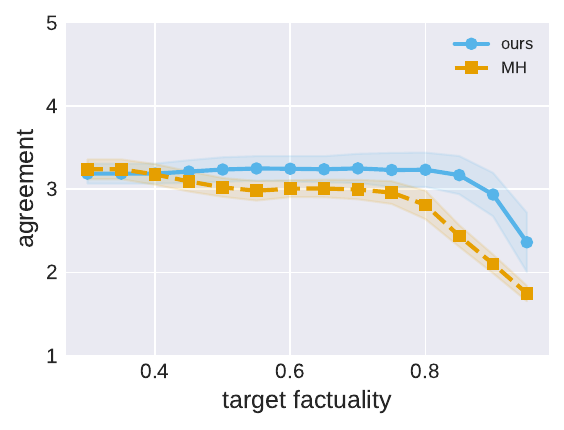}
        \caption{Helpfulness\,$\mid$\,non-abstaining.}
    \end{subfigure}
    \caption{Breakdown of the FActScore Likert results from \Cref{fig:utility-downstream} (panel~b) per statement that the judge is asked to rate its agreement with: (a) whether the biography covers the most important facts a reader would expect to find in a short biography; (b) whether the biography reads as coherent prose with smooth transitions and no abrupt or disjointed jumps; and (c) whether this would be a satisfactory and helpful answer if a user asked for a short biography of the entity. The main-text Likert score is the per-output average over these three. Conditioned on non-abstention; averaged over 10 splits; shaded area shows 95\% CI.}
    \label{fig:utility-likert-axes}
\end{figure}

\subsection{Alternative scoring functions and base models}
\label{appendix:scorer-elicit-vs-logprob}

\begin{figure}[h]
    \centering
    \begin{subfigure}[b]{0.45\textwidth}
        \includegraphics[width=\linewidth]{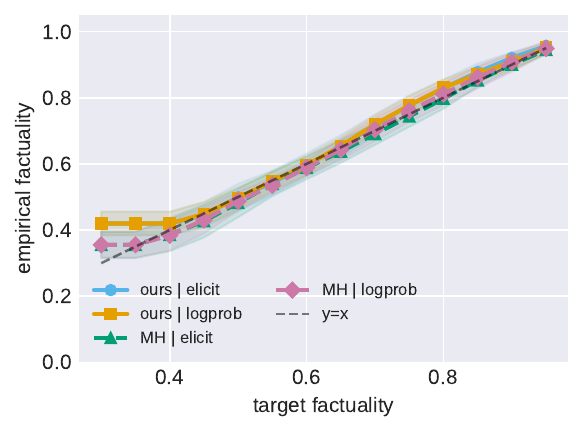}
        \caption{Target vs. observed factuality.}
    \end{subfigure}\hfill
    \begin{subfigure}[b]{0.45\textwidth}
        \includegraphics[width=\linewidth]{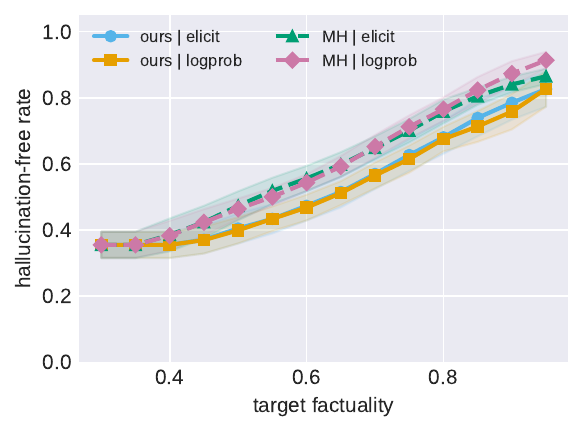}
        \caption{Strict factuality of non-abstaining outputs.}
    \end{subfigure}
    \caption{Scorer-elicitation ablation on FActScore: confidence-rating elicitation (main text) vs.\ raw token log-probabilities. Averaged over 10 splits; shaded area shows 95\% CI.}
    \label{fig:scorer-elicit-vs-logprob}
\end{figure}

The main text elicits per-claim scores for Llama-3.3-70B-Instruct generations by asking GPT-4o to rate its confidence on a $0$-$1$ scale (\Cref{sec:instantiating-scoring}).
A natural, cheaper alternative is to read off the scorer's own token log-probabilities for a forced \texttt{T}/\texttt{F} judgement; we compute this scorer as $s = (p_\texttt{T} + p_\texttt{N})/(p_\texttt{T} + p_\texttt{N} + p_\texttt{F})$ using GPT-4o.
\Cref{fig:scorer-elicit-vs-logprob} shows the impact that this alternative formulation has on calibration validity and sample utility. Both scorers produce empirically valid calibration curves, as predicted.
Furthermore, for our method both scorers yield similar utilities; for \texttt{MH}, the logprob-based scorer yields slightly higher strict factuality among non-abstaining outputs at high targets in this experiment.

In \Cref{fig:scorer-strength}, we perform a related comparison: replacing both the base model and the scorer with smaller (and thus cheaper, but less capable) variants within the same family.
Specifically, we gather generations from Llama-3.1-8B-Instruct \citep{grattafiori2024llama}, which we then score using GPT-4o-mini instead of GPT-4o.
We fix the rest of the experimental setting as in \Cref{sec:experiments}.

As one can observe, validity is agnostic to the capacity of the base model and the scorer, while downstream utility can be affected somewhat, especially at higher target rates.
Overall, however, these differences are small, indicating that calibrated factuality successfully translates to downstream utility.

\begin{figure}[h]
    \centering
    \begin{subfigure}[b]{0.45\textwidth}
        \includegraphics[width=\linewidth]{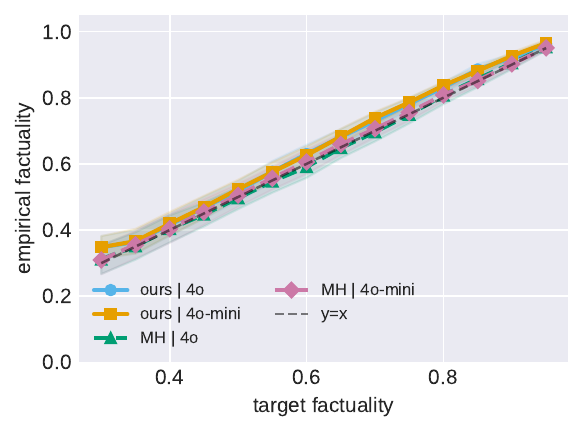}
        \caption{Target vs. observed factuality.}
    \end{subfigure}\hfill
    \begin{subfigure}[b]{0.45\textwidth}
    \includegraphics[width=\linewidth]{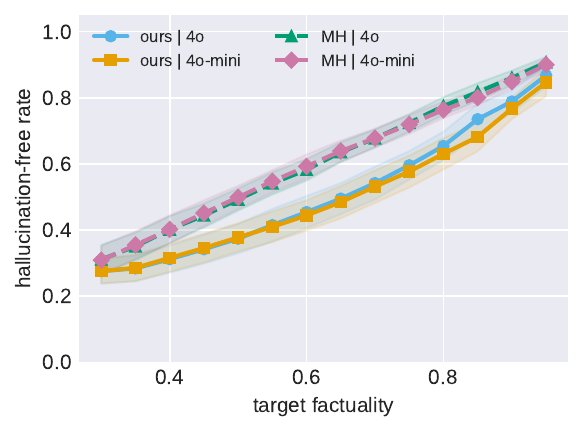}
        \caption{Strict factuality of non-abstaining outputs.}
    \end{subfigure}
    \caption{Model capacity ablation on MATH: GPT-4o vs.\ GPT-4o-mini scoring, all on Llama-3.1-8B-Instruct generations (in contrast to the larger Llama-3.3-70B-Instruct used elsewhere). Averaged over 10 splits; shaded area shows 95\% CI.}
    \label{fig:scorer-strength}
\end{figure}

\subsection{Ablating $\beta$ at $M = 20$ with fixed-budget rejection sampling}
\label{appendix:rs-ablation}

This appendix probes how our finite-particle guarantee for rejection sampling (Lemma~\ref{lemma:glivenko} and Corollary~\ref{corollary:limsup}) interacts with $\beta$ and a fixed-budget rejection sampling deployment procedure.
We fix $M = 20$ as in the main experiments, but at deployment we run rejection sampling up to a hard cap of $50$ attempts; if a non-abstaining, non-accepted generation has not been reached within this budget, we return a fallback sample.
We consider two fallback choices that intuitively give upper and lower bounds on the unknown behavior of the true mixture posterior:
\begin{itemize}
    \item \textbf{Worst-case fallback}: a synthetic sample consisting of a single claim which is incorrect but has a score $1.0$, contributing maximally to the risk estimate. This deliberately pessimistic view assumes that every budget-exhausted prompt would produce a fully false, accepted output.
    \item \textbf{Abstention fallback}: when the budget is exhausted, the \texttt{abstain} response is deterministically sampled. Since this contributes zero to the loss/risk estimate, this is an optimistic view motivated by the observation that the mixture prior would have eventually emitted an abstention had we kept drawing. This fallback would also likely be the choice of practitioners using budget-constrained sampling techniques, especially in deployment settings where abstention is favored over returning a possibly incorrect generation.
\end{itemize}
We contrast these two fallbacks as the true behavior under indefinite rejection sampling lies somewhere between the two.
By comparing the two, we can thus get a sense of how far the asymptotic guarantee of Corollary~\ref{corollary:limsup} can be from being empirically valid at $M = 20$ under different values of $\beta$. However, note that fully characterizing the precise support of the posterior is not possible with finite sample sizes; these fixed-budget fallbacks only give a coarse and loose approximation.

\begin{figure}[h]
    \centering
    \begin{subfigure}[b]{0.32\textwidth}
        \includegraphics[width=\linewidth]{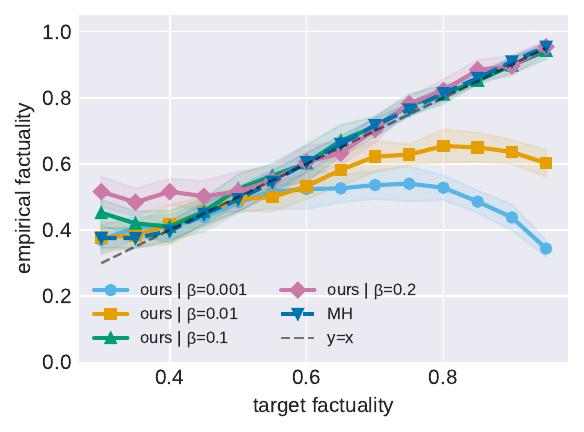}
        \caption{Target vs.\ observed factuality (worst-case fallback).}
        \label{fig:rs-ablation-calibration}
    \end{subfigure}\hfill
    \begin{subfigure}[b]{0.32\textwidth}
        \includegraphics[width=\linewidth]{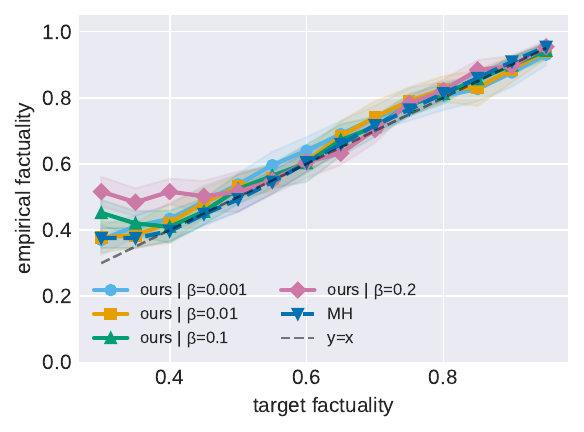}
        \caption{Target vs.\ observed factuality (abstention fallback).}
        \label{fig:rs-ablation-calibration-best}
    \end{subfigure}\hfill
    \begin{subfigure}[b]{0.32\textwidth}
        \includegraphics[width=\linewidth]{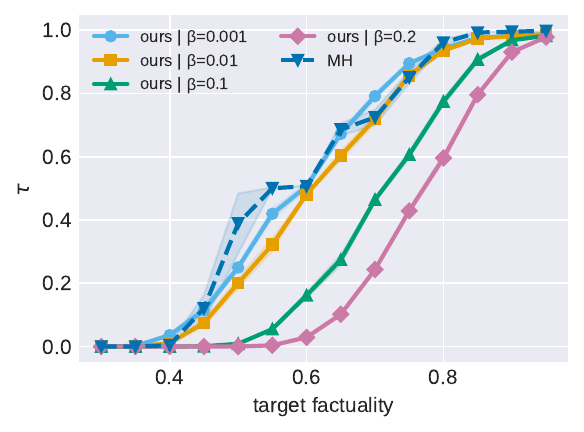}
        \caption{Calibrated threshold $\hat\tau$ (does not depend on the fallback).}
        \label{fig:rs-ablation-tau}
    \end{subfigure}
    \caption{Fixed-budget rejection-sampling $\beta$-ablation on FActScore. (a, b) Target vs.\ empirical factuality under the worst-case (fully-false) and abstention fallbacks; (c) Calibrated $\hat\tau$ vs.\ target (identical for both fallbacks, since calibration is unaffected). Averaged over 10 splits; shaded area shows 95\% CI.}
    \label{fig:rs-ablation-validity}
\end{figure}

\xhdr{Calibration validity}
\Cref{fig:rs-ablation-calibration,fig:rs-ablation-calibration-best} show target factuality $1 - \alpha$ against observed factuality on the held-out test split for different values of $\beta$, under the worst-case and abstention fallbacks respectively.
At a relatively low calibration budget of $M = 20$, as $\beta \downarrow 0$ the calibration guarantee begins to fall apart under the worst-case fallback, yielding undercalibrated responses at high targets due to an increasing fraction of prompts hitting the budget cap; at $\beta = 0.1$ and $\beta = 0.2$, coverage is empirically valid and tight. These correspond to prompts for which the language model has a hard time generating responses that pass the scorer threshold. We conjecture that with a reliable scorer, we would eventually return an abstention particle, but this is impossible to verify under this sampling design.
Under the abstention fallback, by contrast, coverage holds across the whole range of $\beta$.

\begin{figure}[h]
    \centering
    \begin{subfigure}[b]{0.24\textwidth}
        \includegraphics[width=\linewidth]{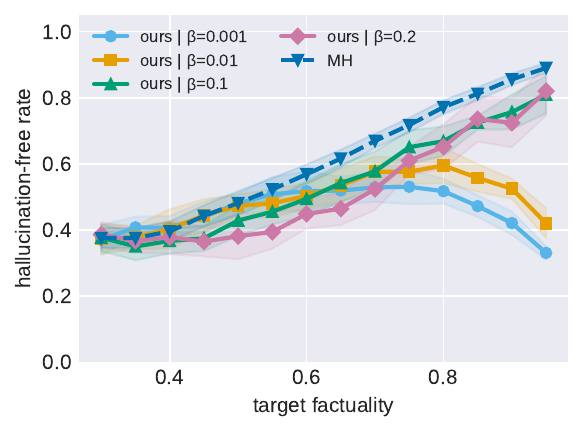}
        \caption{Strict fact.\,$\mid$\,non-abst.\ (worst-case).}
    \end{subfigure}\hfill
    \begin{subfigure}[b]{0.24\textwidth}
        \includegraphics[width=\linewidth]{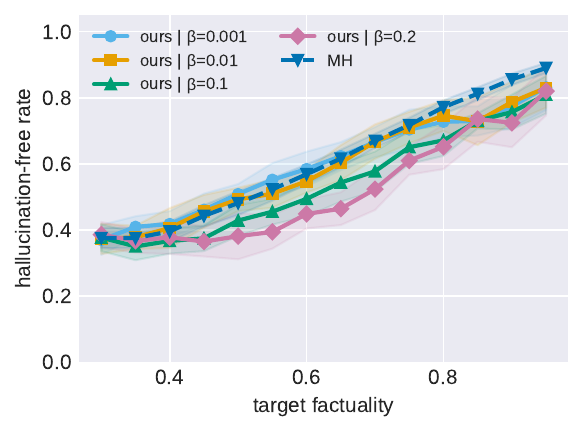}
        \caption{Strict fact.\,$\mid$\,non-abst.\ (abst. fallback).}
    \end{subfigure}\hfill
    \begin{subfigure}[b]{0.24\textwidth}
        \includegraphics[width=\linewidth]{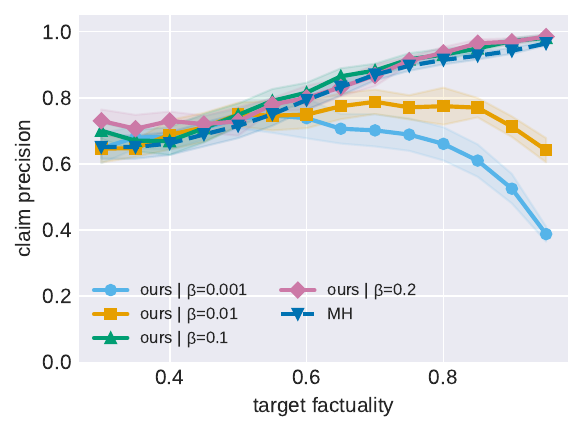}
        \caption{Claim precision (worst-case).}
    \end{subfigure}\hfill
    \begin{subfigure}[b]{0.24\textwidth}
        \includegraphics[width=\linewidth]{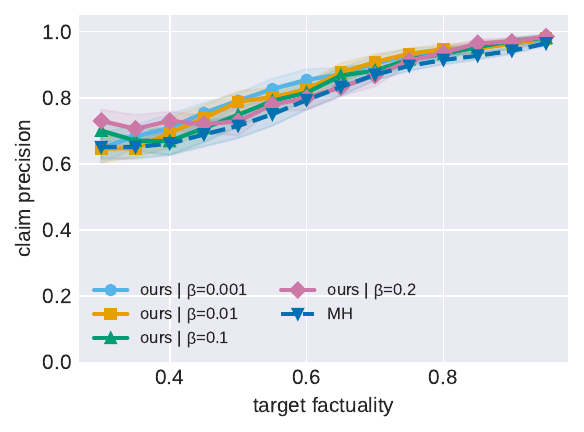}
        \caption{Claim precision (abst. fallback).}
    \end{subfigure}
    \caption{Proxy metrics under the fixed-budget rejection sampling procedure on FActScore, comparing the worst-case (fully-false) and \texttt{abstention} budget-exhausted fallbacks.}
    \label{fig:rs-ablation-proxies}
\end{figure}

\xhdr{Proxy metrics}
\Cref{fig:rs-ablation-proxies} reports two factuality proxy metrics from the main experiments, evaluated under the fixed-budget rejection-sampling procedure rather than the analytic plug-in, again under both fallbacks.
Under the worst-case fallback, the proxy metrics degrade as calibration fails at low $\beta$, mirroring the calibration plots.
Under the abstention fallback, both strict factuality among non-abstaining samples and claim precision remain high across the whole range of $\beta$, since the budget-exhausted prompts do not affect the metrics.
Generally speaking, except for in the instances of worst-case fallback-induced miscalibration, 
lower $\beta$ leads to a closer gap in strict factuality among non-abstaining samples; this is because a lower forced probability of abstention in the mixture prior yields higher values of $\hat\tau$ (\Cref{fig:rs-ablation-tau}), thus applying stricter checks on the outputs.

\section{Prompts}
\label{appendix:prompts}

\subsection{Generation prompts}
\label{appendix:generation-prompts}

We do not set a special system prompt for generation; the dataset-specific prompts below are passed as the user's first turn.

\xhdr{FActScore}
\begin{quote}\small
You are a helpful assistant that writes biographies. You will be given the name of a well-known entity, such as a historical figure, scientist, artist, or music group. Write a concise biography of that entity. Include only facts you are confident are correct, and omit details you are uncertain about. Brief biographies are fine, do not pad with guesses. Keep each sentence focused on a single fact. Prefer short, atomic sentences over long sentences with multiple clauses.

Example:\\
Tell me a bio of Albert Einstein.\\
Albert Einstein was a famed physicist. He was born in Germany in 1879. Einstein is best known for developing the theory of relativity. He published the equation E=mc\textsuperscript{2} in 1905. In 1921, he won the Nobel Prize in Physics. He died in 1955.

Now:\\
Tell me a bio of \{entity\}.
\end{quote}

Where the \emph{entity} is the entity's name, as given in the dataset.

\xhdr{MATH}
\begin{quote}\small
You are a mathematical reasoning assistant. Solve problems step by step in a concise manner. IMPORTANT: Use exactly one sentence per ATOMIC logical reasoning step, NEVER chain more than one step in a sentence. Only include reasoning steps which you are confident are correct.

Example:\\
Problem: Find the value of $x$ such that $\sqrt{x + 7} = 9$.\\
Solution:\\
We need to solve for x where the square root of (x + 7) equals 9.\\
Squaring both sides gives x + 7 = 81.\\
Subtracting 7 from both sides yields x = 74.\\
We verify: sqrt(74 + 7) = sqrt(81) = 9.\\
Therefore, x = 74.

Now solve:\\
Problem: \texttt{\{raw\_prompt\}}\\
Solution:
\end{quote}

Where the $\emph{raw\_prompt}$ is the textual specification of the problem, as given in the dataset.

\subsection{Data labeling prompts}
\label{appendix:annotation-prompts}

\xhdr{FActScore} The biography is segmented into one claim per line and the entity's Wikipedia article is included in-context as the source of truth.
\begin{quote}\small
[TASK DESCRIPTION]\\
You are a fact-checking assistant, assessing a biography written about a notable person or entity.
You will be supplied the name of the entity, a Wikipedia article about that entity, and a biography.
The biography consists of a series of claims, conversational remarks or other statements, given line by line.

Your task is to classify each line following these guidelines:
\begin{itemize}
    \item \textbf{true}: The line makes one or more factual claims about the entity, and EVERY claim is supported by the Wikipedia article.
    \item \textbf{false}: The line makes one or more factual claim about the entity, but AT LEAST ONE claim is contradicted by the Wikipedia article.
    \item \textbf{neutral}: The line is purely conversational or otherwise does not make any factual claims about the requested entity.
    \item \textbf{refusal}: The line professes ignorance about the entity on the speaker's part, it expresses that the speaker is uncertain about who the name might be referring to, or the speaker is indicating that they do not have enough information to say which one out of multiple entities is being referred to.
\end{itemize}
IMPORTANT: All lines containing factual claims about the entity MUST be evaluated as either `true' or `false'.

[ENTITY] \texttt{\{entity\}}\\
{}[WIKIPEDIA ARTICLE] \texttt{\{wikipedia\_content\}}\\
{}[BIOGRAPHY TO EVALUATE] \texttt{\{claims\_text\}}\\
{}[RESPONSE FORMAT] You must return an answer which follows the given JSON schema, with exactly \texttt{N} annotations, one for each line in order.
\end{quote}

\xhdr{MATH}
\begin{quote}\small
You are evaluating a mathematical proof/solution.

Problem: \texttt{\{problem\}}

Reference solution: \texttt{\{reference\_solution\}}

Generated solution steps to annotate (claims start with \texttt{[CLAIM \#]}):

\texttt{\{claims\_text\}}

For each step, classify as:
\begin{itemize}
    \item \textbf{true}: If the step is a mathematically correct derivation, \emph{assuming} the previous steps are correct. Claims can be true without significantly advancing the proof.
    \item \textbf{false}: If the step contains an error, or contradicts the previous steps.
    \item \textbf{neutral}: use for neutral claims. For instance when all of the claims in the output together are expressing uncertainty or inability to solve the problem (e.g.\ ``I am not able to solve\dots''), label each one of them as neutral. Use neutral too for steps that are exclusively filler text (e.g.\ ``Let me answer this for you\dots'').
\end{itemize}

Also determine:
\begin{itemize}
    \item \textbf{complete}: Does the proof reach a final answer AND is the final answer correct (matches reference)?
\end{itemize}

IMPORTANT: You MUST return exactly \texttt{N} annotations, one for each step in order.
\end{quote}

\subsection{Score function prompts}
\label{appendix:scoring-prompts}

We here list the prompts corresponding to the \emph{full-response elicitation} strategy, that is the strategy described in \Cref{sec:instantiating-scoring}.
We also give the prompt used for the alternative logprob-based score function used in Appendix~\ref{appendix:scorer-elicit-vs-logprob}.

\xhdr{FActScore (full-response elicitation)}

\begin{quote}\small
You will be given a context (e.g., a question) and a statement made in response to that question.
Your job is to rate this statement on a scale from 0.0 to 1.0.
The score should represent your confidence in the claims made by the statement, where a 1 is obvious facts and results like `The earth is round' and `1+1=2'. A 0 is for claims that are objectively false, like ``the earth is flat''.
On the other hand, if you're unsure, you should assign probabilities that are strictly in between 0 and 1.
If the statement is not a factual claim (e.g., ``I don't know'', ``I cannot answer that'', ``I'm sorry but\dots'', questions, or expressions of uncertainty), assign a score of 1.0 since these are not false claims.
You may think out loud first, but you MUST end your response with ``Score: X.XX'' where X.XX is a number between 0.0 and 1.0.

\#\#\# The context:\\
\texttt{\{prompt\}}

\#\#\# Previously accepted facts (assume these are correct):\\
\texttt{\{prior\_claims\}}

\#\#\# The statement:\\
\texttt{\{new\_claim\}}

\#\#\# Your response:
\end{quote}

\xhdr{MATH (full-response elicitation)}
\begin{quote}\small
You are evaluating a step in a mathematical derivation.

You are given:
\begin{enumerate}
    \item A math problem
    \item Previous reasoning steps S
    \item A new step C
\end{enumerate}

Give your best estimate of the probability that the new step C is true, using S as context. In other words, you should estimate P(C correct $|$ S). Rate your confidence on a scale from 0.00 to 1.00, and make sure your uncertainty estimates are calibrated, meaning, that the average rate of correctness should correspond to your prediction, conditional on your prediction.
Reserve 1.00 for claims you are certain are correct, except for:
\begin{itemize}
    \item If the step is not a mathematical claim (e.g., ``Let me think\dots'', ``I'll solve this by\dots''), assign 1.00.
    \item Similarly, if the step is an abstention (e.g.\ ``I do not know how to solve this''), assign 1.00.
\end{itemize}

Important: return confidence to two decimal digits!

Think briefly, then end with ``Score: X.XX'' where X.XX is between 0.00 and 1.00.

\#\#\# Math Problem:\\
\texttt{\{prompt\}}

\#\#\# Previous steps (assume correct):\\
\texttt{\{prior\_claims\}}

\#\#\# New step to evaluate:\\
\texttt{\{new\_claim\}}

\#\#\# Your evaluation:
\end{quote}

\xhdr{MATH (logprob elicitation)}
\begin{quote}\small
You are evaluating a step in a mathematical derivation.
Reply with only a single letter. You will get a list of prior claims, and a new step. Reply:
\begin{itemize}
    \item \textbf{Y} = Yes: The step is mathematically valid given the prior steps.
    \item \textbf{N} = No: The step contains an error or contradicts the previous steps.
    \item \textbf{C} = Conversation: the step is not a mathematical claim (e.g., ``Let me think about this carefully.'', ``I'll try a different approach.'') OR an expression of uncertainty (``I'm not sure about this.'', ``I don't know how to solve this problem.'').
\end{itemize}
Reply with exactly one letter: Y, N, C.

Problem: \texttt{\{prompt\}}

Prior steps:\\
\texttt{\{prior\_claims\}}

New step: \texttt{\{new\_claim\}}
\end{quote}

\subsection{Downstream evaluation prompts}
\label{appendix:judge-prompts}

\subsubsection{FActScore: Likert scale}
\label{appendix:judge-likert}

\xhdr{System prompt}
\begin{quote}\small
You are a careful, impartial judge rating short encyclopedia-style biographies. For each statement below, indicate your level of agreement on a 1-5 scale: 1 = strongly disagree, 2 = disagree, 3 = neither agree nor disagree, 4 = agree, 5 = strongly agree.
\end{quote}

\xhdr{User prompt}
\begin{quote}\small
Rate the following biography of \texttt{\{entity\}} on three independent axes.

For each statement, return an integer from 1 (strongly disagree) to 5 (strongly agree):

\begin{enumerate}
    \item \textbf{completeness}: ``The biography covers the most important facts about \texttt{\{entity\}}, as one would expect to find in a short biography.''
    \item \textbf{fluency}: ``The biography reads as coherent prose: sentences flow smoothly into one another, with no awkward phrasing, abrupt transitions, redundancy, or disjointed jumps between unrelated points.''
    \item \textbf{helpfulness}: ``If a reader asked for a short biography of \texttt{\{entity\}}, this would be a satisfactory and helpful answer.''
\end{enumerate}

Biography:
\begin{verbatim}
"""
{text}
"""
\end{verbatim}
\end{quote}

\subsubsection{MATH: Completeness}
\label{appendix:judge-completeness}

We re-use the data labeling prompt from Appendix~\ref{appendix:annotation-prompts}.
The downstream evaluation only consumes the \texttt{complete} field, discarding the rest of the output.

\end{document}